\pdfoutput =1
\newif\ifdraft
\drafttrue  
\draftfalse 

\documentclass[letterpaper]{article}

\usepackage{chngcntr} 
\usepackage[numbers]{natbib} 
\usepackage{subfigure} 
\usepackage[english]{babel}
\usepackage[utf8x]{inputenc}
\usepackage[T1]{fontenc} 

\usepackage[top=3cm,bottom=2cm,left=3cm,right=3cm,marginparwidth=1.75cm]{geometry}

\usepackage[pagebackref=false,breaklinks=true,letterpaper=true,colorlinks=true,bookmarks=false, allcolors=blue, hyperindex,pdftex]{hyperref}

\usepackage[colorinlistoftodos]{todonotes}

\usepackage{url}            
\usepackage{booktabs}       
\usepackage{amsfonts}       
\usepackage{nicefrac}       
\usepackage{microtype}      
\usepackage{lipsum}
\usepackage{xcolor}
\usepackage{enumitem}
\usepackage[english]{babel}
\usepackage{epsfig}
\usepackage{graphicx}
\usepackage{amsmath,amsthm}
\usepackage{amssymb}
\usepackage{algorithm}
\usepackage{algorithmicx}
\usepackage{algpseudocode}
\usepackage{enumitem}
\usepackage{xspace}
\usepackage{thmtools,thm-restate} 
\usepackage{comment}
\usepackage{appendix}

\allowdisplaybreaks

\newtheorem{theorem}{Theorem}[section]
\newtheorem{lemma}[theorem]{Lemma}
\newtheorem{corollary}[theorem]{Corollary}
\newtheorem{proposition}[theorem]{Proposition}

\newenvironment{hproof}{\proof}{\endproof}

\newcommand{\mb}{\boldsymbol}

\newcommand{\bb}{\mathbb}

\newcommand{\vem}{\text{LS-EM\xspace}\xspace}
\newcommand{\vemlong}{\text{Least Squares EM}\xspace}
\newcommand{\variant}{\text{least-squares}\xspace}
\newcommand{\Variant}{\text{Least-squares}\xspace}

\newcommand{\dup}[1]{\ensuremath{\operatorname{d}\!{#1}}} 

\newcommand{\argmin}{\operatornamewithlimits{argmin}}
\newcommand{\argmax}{\operatornamewithlimits{argmax}}

\newcommand{\reals}{\bb R}
\newcommand{\R}{\reals}

\newcommand{\Brac}[1]{\left\lbrace #1 \right\rbrace}

\newcommand{\wh}{\widehat}

\newcommand{\wt}{\widetilde}
\newcommand{\ol}{\overline}

\renewcommand{\tilde}{\wt}
\renewcommand{\hat}{\wh}
\renewcommand{\bar}{\ol}

\definecolor{darkgreen}{RGB}{0,120,0}

\newcommand{\yc}[1]{\ifdraft{\color{darkgreen}{\bf Yudong: #1}}\else\fi}

\title{Global Convergence of \vemlong for Demixing Two Log-Concave Densities}

\author{%
  Wei Qian, Yuqian Zhang, Yudong Chen  \\
  Cornell University\\
\texttt{\{wq34,yz2557,yudong.chen\}@cornell.edu} }

\date{}

\begin{document}

\maketitle

\begin{abstract}

This work studies the location estimation problem for a mixture of two rotation invariant log-concave densities. We demonstrate that \vemlong, a variant of the EM algorithm, converges to the true location parameter from a randomly initialized point. We establish the explicit convergence rates and sample complexity bounds, revealing their dependence on the signal-to-noise ratio and the tail property of the log-concave distribution. Moreover, we show that this global convergence property is robust under model mis-specification. 

Our analysis generalizes previous techniques for proving the convergence results for Gaussian mixtures. In particular, we make use of an angle-decreasing property for establishing global convergence of \vemlong beyond Gaussian settings, as $ \ell_2 $ distance contraction no longer holds globally for general log-concave mixtures.

\end{abstract}

\section{Introduction}
\label{sec:intro}
One important problem in statistics and machine learning is to learn a finite mixture of distributions~\cite{lindsay1995mixture,titterington1985statistical}. 
In the parametric setting where the functional form of the distribution is known, this problem is to estimate parameters (e.g., mean and covariance) that specify the distribution of each mixture component. The parameter estimation problem for mixture models is inherently nonconvex, posing challenges for both computation and analysis. While many algorithms have been proposed, rigorous performance guarantees are often elusive. One exception is the Gaussian Mixture Model (GMM), for which much theoretical progress has been made in recent years. The goal of this paper is to study algorithmic guarantees for a much broader class of mixture models, namely log-concave distributions. This class includes may common distributions\footnote{Familiar examples of log-concave distributions include Gaussian, Laplace, Gamma, and Logistics~\cite{bagnoli2005log}. } and is interesting from both modelling and theoretical perspectives \cite{an1997log,bagnoli2005log,barlow1975statistical,gilks1992adaptive,walther2002detecting,saumard2014log}.  

We focus on the Expectation Maximization (EM) algorithm~\cite{dempster1977maximum}, which is one of the most popular methods for estimating mixture models. Understanding the convergence property of EM is highly non-trivial due to the non-convexity of the negative log-likelihood function. The work in~\cite{balakrishnan2017statistical} developed a general framework for establishing \emph{local} convergence to the true parameter. Proving \emph{global} convergence of EM is more challenging, even in the simplest setting with a mixture of two Gaussians (2GMM). The recent work in~\cite{daskalakis2016ten,xu2016global} considered balanced 2GMM with known covariance matrix and showed for the first time that EM converges to the true location parameter using random initialization. Subsequent work established global convergence results for a mixture of two truncated Gaussians~\cite{nagarajan2019convergence}, two linear regressions (2MLR)~\cite{kwon2018global,klusowski2019estimating}, and two one-dimensional Laplace distributions~\cite{barazandeh2018behavior}.

All the above results (with the exception of~\cite{barazandeh2018behavior}) rely on the explicit density form and specific properties of the Gaussian distribution.  
In particular, under the Gaussian distribution, the M-step in the EM algorithm has a closed-form expression, which allows a straightforward analysis of the convergence behavior of the algorithm. However, for general log-concave distributions, the M-step no longer admits a closed-form solution, which poses significant challenges for analysis. To address this difficulty, we consider a modification of the standard EM algorithm, \textit{\vemlong} (\vem), for learning the location parameter of a mixture of two log-concave distributions. The \vem algorithm admits a simple, explicit update rule in the M-step.

As the main result of this paper, we show that for a mixture of rotation invariant log-concave distribution, \vem converges to the true location parameter from a randomly initialized point. Moreover, we provide explicit convergence rates and sample complexity bounds, which depend on the signal-to-noise ratio as well as the tail property of the distribution. As the functional form of the true density may be unknown, we further establish a robustness property of \vem when using a mis-specified density. As a special case, we show that using a Gaussian distribution, \vem globally converges to a solution close to the true parameter whenever the variance of the true log-concave density is moderate.

\paragraph{Technical Contributions} 

We generalize the sensitivity analysis in~\cite{daskalakis2016ten} to a broad class of log-concave distributions. In the process, we demonstrate that log-concavity and rotation invariance of the distribution are the only properties required to guarantee the global convergence of \vem. Moreover, our analysis highlights the fundamental role of an \emph{angle-decreasing} property in establishing the convergence of \vem to the true location parameter in the high dimension settings. Note that contraction in the $\ell_2$ distance, upon which the previous convergence results were built, no longer holds globally for general log-concave mixtures.

\paragraph{Organization}
In Section~\ref{sec:background}, we formulate the parameter estimation problem for a mixture of log-concave distributions and review related work. In Section~\ref{sec:algo}, we delineate the \vemlong algorithm and elucidate its connection with classical EM. Analysis of the global convergence of \vem is provided in Section~\ref{sec:analysis} under the population setting. Finite-sample results are presented in Section~\ref{sec:finite}, with Section~\ref{sec:misspecification} dedicated to the model mis-specification setting. The paper concludes with a discussion of future directions in Section~\ref{sec:conclusion}. Some details of the proofs are deferred to the Appendix.

\paragraph{Notations}
We use $x\in\R$ and $\mb x\in\R^d$ to denote scalars and vectors, respectively;  $X\in\R$ and $\mb X\in\R^d$ to denote scalar and vector random variables, respectively. The $i$-th coordinate of $\mb x$ (or $\mb X$) is $x_i$ (or$X_i$), and the $j$-th data point is denoted by $\mb x^j$ or  $\mb X^j$. The Euclidean norm in $\bb{R}^d$ is $\|\cdot\|_2$. For two vectors $\mb\alpha,\mb\beta \in \bb{R}^d$, we use $\angle (\mb\alpha,\mb\beta)\in(0,\pi)$ to denote the angle between them, and $\langle \mb\alpha,\mb\beta\rangle $ to denote their inner product. Finally, $\mb I_d$ is a $d$-by-$d$ identity matrix.

\section{Problem Setup}
\label{sec:background}

In this section, we set up the model for a mixture of log-concave distributions, and discuss the corresponding location estimation problem in the context of existing work.

\subsection{Data Generating Model}

Let $\mathcal{F}$ be a class of rotation invariant log-concave densities in $\mathbb{R}^d$ defined as follows:
\begin{equation}\label{eq:log-concave class}
\begin{aligned}
    \mathcal{F} = \bigg\{f: &f(\mb x) = \frac{1}{C_g}\exp\big(-g(\|\mb x\|_2)\big), \text{$g$ is convex and strictly increasing on $[0,\infty)$}, \\
    & \int f(\mb x) \dup {\mb x}= 1, \int x_i^2 f(\mb x) \dup {\mb x} = 1, \forall i \in [d] \bigg\}. 
\end{aligned}
\end{equation}
Without loss of generality, we assume $g(\mb 0)=0$.\footnote{Note that $\mb x \mapsto g(\|\mb x\|_2)$ is a convex function,  as it is the composition of a convex function and a convex increasing function. The normalization constant $C_g$ can be computed explicitly by
$
C_g =  C_h d v_d
$ 
with $ C_h =  \int_{0}^{\infty} t^{d-1} \exp(-g(t)) \dup t,$
where $v_d:=\frac{\pi^{d/2}}{\Gamma(d/2+1)}$ is the volume of a unit ball in $\bb{R}^d$.} 
It can be verified that each $f \in \mathcal{F}$ has mean $\mb 0$ and covariance matrix $\mb I_d$. For each $f\in \mathcal{F}$, we may generate a location-scale family consisting of the densities $f_{\mb\beta,\sigma} (\mb x): = \frac{1}{\sigma^d}f\big(\frac{\mb x-\mb\beta}{\sigma}\big)$, which has mean $\mb\beta$ and covariance matrix $\sigma^2 \mb I_d$. 

We assume that each data point $\mb X^i, i\in[n]$ is sampled independently from the distribution $D(\mb\beta^*,\sigma)$, defined as a balanced mixture of two densities from the above log-concave location-scale family:
\begin{align}
\label{eq:generative_process}
    D(\mb\beta^*,\sigma):= \frac{1}{2} f_{\mb\beta^*,\sigma} + \frac{1}{2} f_{-\mb\beta^*,\sigma}.
\end{align}
It is often useful to view this mixture model as an equivalent latent variable model: independently for each $i\in[n]$, an unobserved label $Z_i \in \left\{1,2\right\}$ is first generated according to
\begin{align*}
\mathbb{P}(Z_i=1) = \mathbb{P}(Z_i=2) = \frac{1}{2},
\end{align*}
and then the data point $\mb X^i$ is sampled from the corresponding mixture component, i.e., from $f_{\mb \beta^*,\sigma} $ if $Z_i=1$ and from $f_{-\mb \beta^*,\sigma} $ otherwise. 

Since $ \{ f_{\mb \beta,\sigma}\}  $ is a location-scale family, the above generative process can be equivalently written as
\begin{align*}
\mb X^i =
\begin{cases}
 \mb \beta^* + \sigma \mb E^i, & \text{if } Z_i = 1,\\
 -\mb \beta^* + \sigma \mb E^i, & \text{if } Z_i = 2,
\end{cases}
\end{align*}
where $ \mb E^i \overset{\text{i.i.d.}}{\sim} f $ can be viewed as the additive noise.
This equivalent representation motivates us to define the signal-to-noise ratio (SNR) 
\begin{align} \label{eq:snr}
\eta: = \frac{\|\mb \beta^*\|_2}{\sigma},
\end{align}
which is used throughout this paper.

\paragraph{Examples:} 

Below are several familiar examples of one-dimensional log-concave distributions $f \propto \exp(-g)$ from $\mathcal{F}$:
\begin{enumerate}
    \item \emph{Polynomial distributions}: $ g(x) \propto |x|^r$ with $ r\geq 1$.
   When $r=2$, it corresponds to the \emph{Gaussian} distribution. When $r=1$, it corresponds to the \emph{Laplace} distribution. 
    \item \emph{Logistic distribution}: $g(x) \propto \log(e^{-|x|/2} + e^{|x|/2})$.
\end{enumerate}
These distributions can be generalized to higher dimensional scenarios by replacing $|x|$ with $\|\mb x\|_2$. 
In Appendix~\ref{appendix-log-concave}, we provide a review of some elementary properties of log-concave distributions.

\subsection {Location Estimation and the EM Algorithm}

We assume that $\sigma$ is known, and our goal is to estimate the location parameter $\mb \beta^*$ from data $\mb  X^1,\mb X^2,\ldots, \mb X^n\in \bb{R}^d $ sampled i.i.d.\ from the mixture distribution $D(\mb\beta^*,\sigma)$ as defined in~\eqref{eq:generative_process}. We first consider this problem for a given log-concave family for which the base density $f$ (equivalently, $g$) is known. The case with an unknown $f$ is discussed in Section~\ref{sec:misspecification}.

Since the negative log-likelihood function of the mixture~\eqref{eq:generative_process} is nonconvex, computing the standard MLE for $\mb \beta^*$ involves a nonconvex optimization problem. EM is a popular iterative method for computing the MLE, consisting of an expectation (E) step and a maximization (M) step. In a standard implementation of EM, the E-step computes the conditional distribution of the labels $Z_i$ under the current estimate of $\mb\beta^*$, and the M-step computes a new estimate by maximizing the conditional log-likelihood based on the distribution obtained in the E-step. The \vem algorithm we consider, described in Section~\ref{sec:algo} to follow, is a variant of the standard EM algorithm with a modified M-step.

\subsection{Convergence of EM and Related Work}
\label{sec:related}

Despite the popularity and empirical success of the EM algorithm, our understanding of its theoretical property is far from complete. 
Due to the nonconvexity of negative log-likelihood functions, EM is only guaranteed to converge to a stationary point in general~\cite{wu1983convergence}. Quantitative convergence results only began to emerge in recent years. The work~\cite{balakrishnan2017statistical} proposed a general framework for establishing the \emph{local} convergence of EM when initialized near the true parameter, with applications to 2GMM, 2MLR, and regression with missing coefficients. Extensions to multiple components are considered in~\cite{yan2017convergence}. 

Beyond local convergence, it is known that the likelihood function of GMM may have bad local optima when there are more than two components, and EM fails to find the true parameter without a careful initialization~\cite{jin2016local}. 
Analysis of the global convergence of EM has hence been focused on the two component setting, as is done in this paper. The work in~\cite{daskalakis2016ten,xu2016global} showed that EM converges from a random initialization for 2GMM. Subsequent work in~\cite{kwon2018global,klusowski2019estimating,nagarajan2019convergence} established similar results in other settings, most of which involve Gaussian models.  An exception is \cite{barazandeh2018behavior}, which proved the global convergence of EM for a mixture of 2 Laplace distributions and derived an explicit convergence rate, but only in the one-dimensional population (infinite sample) setting. We also note that the work~\cite{chaudhuri2009kmeans} studied convergence properties of Lloyd's k-means algorithm---a close relative of EM---for Gaussian mixtures. In general, properties of EM for mixtures of other distributions are much less understood, which is the problem we target at in this paper.

The log-concave family we consider is a natural and flexible generalization of  Gaussian. This family includes many common distributions, and has broad applications in economics~\cite{an1997log,bagnoli2005log}, reliability theory~\cite{barlow1975statistical} and sampling analysis~\cite{gilks1992adaptive}; see~\cite{walther2002detecting,saumard2014log} for a further review. Existing work on estimating log-concave distributions and mixtures has mostly considered the non/semi-parametric setting~\cite{walther2002detecting,cule2010theoretical,jongbloed1998iterative,rufibach2006log,cule2010theoretical,diakonikolas2018polynomial}; these methods are flexible but typically more computational and data intensive than the parametric approach we consider. Other approaches of learning general mixtures include spectral methods~\cite{achlioptas2005spectral,sanjeev2001learning} and tensor methods~\cite{hsu2013learning,chaganty2013spectral}, and EM algorithm is often applied to the output of these methods.

\section{The \vemlong Algorithm}
\label{sec:algo}
As mentioned, the M-step in the standard EM involves maximizing the conditional log-likelihood. For GMM, the M-step is equivalent to solving a least-squares problem. While for a mixture of log-concave distributions, the M-step is equivalent to solving a convex optimization problem, and this optimization problem does not admit a closed form solution in general. This introduces complexity for both computation and analysis. 

We instead consider \vemlong (\vem), a variant of EM that solves a least-squares problem in the M-step even for non-Gaussian mixtures. To elucidate the algorithmic property, we first consider \vem in the population setting, where we have access to an infinite number of data sampled from the mixture distribution $D(\mb\beta^*,\sigma)$. The finite sample version is discussed in Section \ref{sec:finite}. 

Each iteration of the population \vem algorithm consists of the following two steps:
\begin{itemize}[leftmargin = 5mm]
\item \textbf{E-step:} Compute the conditional probabilities 
of the label $Z\in \left\{1,2\right\}$ given the current location estimate $\mb \beta$: 
\begin{align}
\label{e-step}
p^1_{\mb\beta,\sigma}(\mb X): =\frac{f_{\mb\beta,\sigma}(\mb X)}{f_{\mb\beta,\sigma}(\mb X)+f_{-\mb\beta,\sigma}(\mb X)},
\qquad
p^2_{\mb\beta,\sigma}(\mb X): =\frac{f_{-\mb\beta,\sigma}(\mb X)}{f_{\mb\beta,\sigma}(\mb X)+f_{-\mb\beta,\sigma}(\mb X)}. 
\end{align}

\item \textbf{\Variant M-step:} Update the location estimate $\mb \beta$ via weighted least squares: 
\begin{equation}\label{eq:vem}
\begin{aligned}
\mb\beta^+ = & \argmin_{\mb b} \bb{E}_{\mb X\sim D(\mb\beta^*,\sigma)} \left[p^1_{\mb\beta,\sigma}(\mb X) \|\mb X-\mb b\|_2^2 + p^2_{\mb\beta,\sigma}(\mb X ) \|\mb X + \mb b\|_2^2\right] \\
= & \mathbb{E}_{\mb X\sim D(\mb\beta^*,\sigma)} \mb X\tanh \Bigg(\frac{1}{2} g\left(\frac{1}{\sigma}\|\mb X+\mb\beta\|_2\right) - \frac{1}{2} g\left(\frac{1}{\sigma}\|\mb X-\mb\beta\|_2\right)\Bigg):=M(\mb\beta^*,\mb\beta).
\end{aligned} 
\end{equation}
\end{itemize}
In \eqref{eq:vem}, we minimize the sum of squared distances of $\mb X$ to each component's location, weighted by the conditional probability of $\mb X$ belonging to that component. One may interpret \vem as a soft version of the K-means algorithm: instead of assigning each $\mb X$ exclusively to one of the components, we assign a corresponding probability computed using the log-concave density. 
\subsection{Connection to Standard EM}

In contrast to \vem, the M-step in the standard EM algorithm involves maximizing the weighted log-likelihood function (or minimizing the weighted negative log-likelihood function):
\begin{align}
&\textbf{Standard M-step:} \nonumber \\
&\qquad \argmax_{\mb b} \; Q(\mb b \mid \mb\beta ) := \bb{E}_{\mb X\sim D(\mb\beta^*,\sigma) } \left[ p^1_{\mb\beta,\sigma}(\mb X) \log f_{\mb b,\sigma}(\mb X) + p^2_{\mb\beta,\sigma}(\mb X) \log f_{-\mb b,\sigma}(\mb X)\right]. \label{m-step}
\end{align}

The standard EM iteration, consisting of~\eqref{e-step} and~\eqref{m-step}, corresponds to a minorization-maximization procedure for finding the MLE under the statistical setting~\eqref{eq:generative_process}. In particular, the function $Q(\cdot \mid \mb\beta)$ above is a lower bound of the (marginal) log-likelihood function of~\eqref{eq:generative_process}, and the standard M-step~\eqref{m-step} finds the maximizer of this lower bound. In general, this maximization can only be solved approximately. For example, the ``gradient EM'' algorithm considered in~\cite{balakrishnan2017statistical} performs one gradient ascent step on the $Q(\cdot \mid \mb\beta)$ function. 

The \variant M-step~\eqref{eq:vem} admits an explicit update. Moreover, it may also be viewed as an approximation to the standard M-step~\eqref{m-step}, as we observe numerically (see Appendix \ref{appendix:numerical-approxm}) that the \vem update $\mb\beta^+$ satisfies
\begin{align}
Q(\mb\beta^+\mid \mb\beta) > Q(\mb\beta \mid \mb\beta ) \quad \text{if } \mb\beta\neq \mb\beta^*.
\end{align}
This observation indicates that the \variant M-step finds an improved solution $\mb\beta^+$ (compared to the previous iterate $\mb\beta$) for function $Q(\cdot \mid \mb\beta).$

\section{Analysis of \vemlong}
\label{sec:analysis}
In this section, we analyze the convergence behavior of the \vem update \eqref{eq:vem} in the population setting. We first consider the one dimensional case ($d=1$) in Section~\ref{sec:analysis-d=1} and establish the global convergence of \vem, extending the techniques in \cite{daskalakis2016ten} for 2GMM to log-concave mixtures. 
In Section~\ref{sec:analysis:d>1}, we prove global convergence in the multi-dimensional case ($d>1$). In this setting, the \vem update is \emph{not}  contractive in $\ell_2$, so the analysis requires the new ingredient of an angle decreasing property. 

For convenience, we introduce the shorthand $F_{\mb\beta,\sigma}(\mb X) := g\left(\frac{1}{\sigma}\|\mb X+\mb\beta\|_2\right)-g\left(\frac{1}{\sigma}\|\mb X-\mb\beta\|_2\right)$; when $\sigma=1$, we simply write $F_{\mb\beta} \equiv F_{\mb\beta, 1}$. Since the integrand in~\eqref{eq:vem} is an even function of $\mb X$, the update~\eqref{eq:vem} can be simplified to an equivalent form by integrating over one component of the mixture:
\begin{align}
\label{eq:one_comp_variant_update}
\mb\beta^+ = M(\mb\beta^*,\mb\beta) = \bb{E}_{\mb X\sim f_{\mb\beta^*,\sigma}} \mb X \tanh\big(0.5F_{\mb\beta,\sigma}(\mb X)\big) .
\end{align}

Throughout the section, we refer to the technical conditions permitting the interchange of differentiation and integration as the \emph{regularity condition}. This condition is usually satisfied by log-concave distributions --- a detailed discussion is provided in Appendix \ref{appendix-subsec:regularity}. 

\subsection{One Dimensional Case ($d=1$)}
\label{sec:analysis-d=1}

For one dimensional log-concave mixtures, the behavior of \vem is similar to that of EM algorithm for 2GMM: there exist only 3 fixed points, $0$, $\beta^*$, and $-\beta^*$, among which $0$ is non-attractive. Consequently, \vem converges to the true parameter ($\beta^*$ or $-\beta^*$) from any non-zero initial solution $\beta^0$. This is established in the following theorem.
\begin{restatable}[Global Convergence, 1D]{theorem}{onedglobal}
\label{thm:Global-Convergence-1D-VEM}
Suppose that $f\in \mathcal{F}$ satisfies the regularity condition. The \vem update \eqref{eq:vem}, $\beta \mapsto M(\beta^{*},\beta)$, has exactly three fixed points: $0$, $\beta^{*}$ and $-\beta^{*}$. Moreover, the following one-step bound holds:
\[
|M(\beta^{*},\beta)-\textup{sign}(\beta\beta^*)\beta^{*}|\leq \kappa(\beta^{*},\beta,\sigma ) \cdot \big| \beta-\textup{sign}(\beta\beta^*)\beta^{*} \big|,
\]
where the contraction factor
\[
\kappa(\beta^{*},\beta,\sigma ) := \mathbb{E}_{X\sim f_{\min(|\beta|,|\beta^{*}|),\sigma}} \left[1-\tanh\left(0.5F_{\min \left(|\beta|,|\beta^*|\right),\sigma}(X)\right)\right]
\]
satisfies $0<\kappa(\beta^{*},\beta,\sigma )<1$ when $\beta\not\in \{0,\beta^*,-\beta^*\}$. 
\end{restatable}
We prove this theorem in Appendix~\ref{subsec:proof-thm-Global-Convergence-1D-VEM}.
The crucial property used in the proof is the {\em self-consistency} of the \vem update \eqref{eq:vem}, namely $M(\beta,\beta)=\beta$ for all $\beta$. This property allows us to extend the sensitivity analysis technique for 2GMM to general log-concave distributions.

It can be further shown that the contraction factor $\kappa(\beta^{*},\beta,\sigma)$ becomes smaller as the iterate approaches the true $\mb \beta^*$ (see Lemma~\ref{lem:derivative-opt-param-1}). We thus obtain the following corollary on global convergence at a geometric rate. Without loss of generality, we assume $\beta^* >0$.
\begin{corollary}[$t$-step Convergence Rate, 1D]
\label{cor:tstep_convergence}
Suppose that $f\in \mathcal{F}$ satisfies the regularity condition. Let $\beta^{t}$ denote the output of \vem after $t$ iterations, starting from $\beta^{0}\neq 0$.
The following holds:
\[
|\beta^{t}-\textnormal{sign}(\beta^0\beta^{*})\beta^{*}|
\leq \kappa(\beta^{*},\beta^{0},\sigma)^{t} \cdot \big|\beta^{0}-\textnormal{sign}(\beta^0 \beta^{*})\beta^{*}\big|.
\]
If  $\beta^0$ is in $(0,0.5\beta^*)$ or $(1.5\beta^*,\infty)$, running \vem for $O\left({\log \frac{0.5\beta^*}{|\beta^0-\beta^*|}}/{\log \kappa (\beta^*,\beta^0,\sigma)}\right)$ iterations outputs a solution in $(0.5\beta^*,1.5\beta^*)$. In addition, if $\beta^0$ is in $ (0.5\beta^*,1.5\beta^*)$, running \vem for $O\left(C_f(\eta)\log(1/\epsilon)\right)$
iterations outputs an $\epsilon$-close estimate of $\beta^*$, where $C_f(\eta)>0$ is a constant depending only on $f$ and the SNR $\eta$.
\end{corollary}

\paragraph{Special cases}
We provide explicit convergence rates for mixtures of some common log-concave distributions. Again, we assume $\beta^* >0$ and $\beta \ge 0$ without loss of generality, and set $z:=\min(\beta,\beta^*)$.
\begin{itemize}[leftmargin = 8mm]
\item
Gaussian: $\kappa(\beta^*,\beta ,\sigma)\leq\exp\big(-{z^{2}}/{2\sigma^{2}}\big)$ and $C_f(\eta)=\max\big(1,\frac{1}{\eta^{2}}\big).$
\item 
Laplace: $\kappa(\beta^{*},\beta,\sigma)\leq\frac{2\exp(-\frac{\sqrt{2}}{\sigma}z)}{1+\exp(-2\frac{\sqrt{2}}{\sigma}z)}$ and $C_f(\eta)=\max\big(1,\frac{1}{\eta}\big).$
\item 
Logistic: $\kappa(\beta^{*},\beta,\sigma)\leq\frac{4\exp(-\frac{\pi z}{\sigma\sqrt{3}})}{1+\exp(-\frac{2\pi z}{\sigma\sqrt{3}})+2\exp(-\frac{\pi z}{\sigma\sqrt{3}})}$ and  $C_f(\eta)=\max\big(1,\frac{1}{\eta}\big)$.
\end{itemize}

See Appendix~\ref{subsec:concrete-convergence-rate} for the proofs of the above results. Note that the convergence rate depends on the signal-to-noise ratio $\eta$ as well as the asymptotic growth rate $\gamma \equiv \gamma_f$ of the log-density function $g = -\log f$. In the above examples, $\kappa(\beta^*,\beta,\sigma) \approx \exp \left(-c ({\min(\beta^*,\beta)}/{\sigma)^{\gamma}}\right)$, where $\gamma=1$ for Laplace and Logistic distributions, and $\gamma=2$ for Gaussian distribution. 

\subsection{High Dimensional Case ($d>1$)}
\label{sec:analysis:d>1}

Extension to higher dimensions is more challenging for log-concave mixtures than for Gaussian mixtures. Unlike Gaussian, a log-concave distribution with diagonal covariance may not have independent coordinates. A more severe challenge arises because \emph{\vem is not contractive in $\ell_2$ distance}  to the true parameter for general log-concave mixtures. This phenomenon, proved in the lemma below, stands in sharp contrast to the Gaussian mixture problem.

\begin{restatable}[Non-contraction in $\ell_2$]{lemma}{noncontractelltwo}
\label{lem: non-contraction}
Consider a log-concave density of the form $g(\mb x) \propto \|\mb x\|_2^r$ with $r \ge 1$. When $r\in [1,2]$, $\mb 0$ is the only fixed point of \vem in the direction ortoghonal to $\mb \beta^*$. When $r\in (2,\infty)$, there exists a fixed point other than $\mb 0$ in the orthogonal direction. Consequently, when $r>2$, there exists $\mb\beta$ such that $\|M(\mb\beta^{*},\mb\beta)-\mb\beta^*\|_2 > \|\mb\beta - \mb\beta^*\|_2$.
\end{restatable}
We prove Lemma~\ref{lem: non-contraction} in Appendix~\ref{subsec:non-contraction}. The lemma shows that it is fundamentally impossible to prove global convergence of \vem solely based on $\ell_2$ distance, which was the approach taken in~\cite{daskalakis2016ten} for Gaussian mixtures.

Despite the above challenges, we show \emph{affirmatively} that \vem converges globally to $\pm \mb\beta^*$ for mixtures of rotation-invariant log-concave distributions, as long as the initial iterate is not orthogonal to~$\mb\beta^*$ (a measure zero set).

As the first step, we use rotation invariance to show that the \vem iterates stay in a two-dimensional space. The is done in the following lemma, with proof in Appendix~\ref{subsec:proof-two-dim-structure}.

\begin{restatable}[\vem is 2-Dimensional]{lemma}{twodimstructure}
\label{lem:two-dim-structure}
The \vem update satisfies: $M(\mb\beta^{*},\mb\beta) \in \text{span}(\mb\beta, \mb\beta^{*})$.
Moreover, if $\angle (\mb\beta, \mb\beta^{*})=0$
or $\angle (\mb\beta, \mb\beta^{*})=\pi/2$, then $M(\mb\beta^{*},\mb\beta) \in \text{span}(\mb\beta)$.
\end{restatable}

We next establish the asymptotic global convergence property of \vem.

\begin{restatable}[Global Convergence, $d$-Dimensional]{theorem}{highdimglobal}
\label{thm:location-high-dim-asymptotic} 
Suppose that $f\in \mathcal{F}$ satisfies the regularity condition. The \vem algorithm converges to $\text{sign}(\langle \mb\beta^0,\mb\beta^*\rangle)\mb\beta^{*}$ from any randomly initialized point $\mb\beta^0$ that is not orthogonal to $\mb\beta^{*}$.
\end{restatable}

We prove the theorem using a sensitivity analysis that shows decrease \emph{in angle} rather than in $\ell_2$ distance to the true parameter. The proof does not depend on the explicit form of the density, but only log-concavity and rotation invariance. 
We sketch the main ideas of proof below, deferring the details to Appendix~\ref{subsec:proof-high-d-theorem}.
\begin{hproof}
Let $\mb\beta^{0}$ be the initial point that is not orthogonal to $\mb\beta^{*}$. Without loss of generality, we assume $\langle\mb\beta^{0},\mb\beta^{*}\rangle>0$. Consequently, all the future iterates satisfy $\langle \mb\beta^t,\mb\beta^*\rangle >0$ (see Lemma \ref{lem:along convegence}). 

If $\mb\beta^{0}$ is in the span of $\mb\beta^{*}$ (i.e., $\mb\beta^{0}$ parallels $\mb\beta^{*}$), Lemma~\ref{lem:optimal-direction-convergence} ensures that the iterates remain in the direction of $\mb\beta^{*}$ and converge to $\mb \beta^{*}$. 
On the other hand, if  $\mb\beta^{0}$ is not in the span of $\mb\beta^{*}$, we make use of the following two key properties of the \vem update $ \mb \beta^+ = M(\mb \beta^*,\mb \beta)$:
\begin{enumerate}[nosep,leftmargin=5mm]
    \item \textit{Angle Decreasing Property} (Lemma \ref{lem:angle-decrease}): Whenever $\angle \mb\beta,\mb\beta^*\in (0,\frac{\pi}{2})$, the \vem update strictly decreases the iterate's angle toward $\mb\beta^*$, i.e., $\angle (\mb\beta^+,\mb\beta^*) < \angle (\mb\beta,\mb\beta^*) $ ;
    \item\textit{Local Contraction Region} (Corollary~\ref{cor:non-escaping}): there is a local region around $\mb\beta^*$ such that if any iterate falls in that region, all the future iterates remain in that region. 
\end{enumerate}
 Since the sequence of \vem iterates is bounded, it must have accumulation points. Using the {angle decreasing property} and the continuity of $M(\mb \beta^*,\mb \beta)$ in the second variable $\mb\beta$, we show that all the accumulation points must be in the direction of $\mb\beta^*$. In view of the dynamics of the 1-dimensional case (Theorem~\ref{thm:Global-Convergence-1D-VEM}), we can further show that the set of accumulation points must fall into one of the following three possibilities: $\left\{\mb0\right\}$, $\left\{\mb\beta^{*}\right\}$, or $\left\{\mb0,\mb\beta^*\right\}$.
Below we argue that $\left\{\mb0\right\}$ and $\left\{\mb0,\mb\beta^*\right\}$ are impossible by contradiction.
\begin{itemize}[nosep, leftmargin = 5mm]
    \item If $\Brac{\mb0}$ is the set of accumulation points, the sequence of non-zero  iterates $\left\{\mb \beta^t \right\}$ would converge to $\mb 0$ and stay in a neighborhood of $\mb0$ after some time $T$; in this case, Lemma~\ref{lem:norm_increasing} states that the norm of the iterates is bounded away from zero in the limit and hence they cannot converge to $\mb 0$. 
     \item If $\Brac{\mb0,\mb\beta^*}$ is the set of accumulation points, then there is at least one iterate in the local region of $\mb\beta^*$;  by the local contraction region property above, all the future iterates remain close to $\mb\beta^*$. Therefore, $\mb0$ cannot be another accumulation point. 
\end{itemize}
At last, we conclude that $\mb\beta^*$ is the only accumulation point, which \vem converges to.
\end{hproof}

\section{Finite Sample Analysis}
\label{sec:finite}

In this section, we consider the finite sample scenario, where we are given $n$ data points $\mb  X^i $ sampled i.i.d.\ from $ D(\mb\beta^*,\sigma)$. Using the equivalent expression~\eqref{eq:one_comp_variant_update} for the population \vem update, and replacing the expectation  with the sample average, we obtain the finite-sample \vem update:\footnote{This expression is for analytic purpose only. To actually implement \vem, we use samples $\mb X^i$ from the mixture distribution $D(\mb \beta^*,\sigma) $, which is equivalent to \eqref{eq:finite-vem-step}.}
\begin{align}
\label{eq:finite-vem-step}
\tilde{\mb \beta}^+ = \frac{1}{n}\sum_{i=1}^{n}\mb X^i \tanh(0.5 F_{\mb \beta,\sigma}(\mb X^i)),
\qquad
\text{where }
\mb X^i \overset{\text{i.i.d.}}{\sim} f_{\mb\beta^*,\sigma} .
\end{align}

One approach to extend the population results (in Section \ref{sec:analysis}) to this case is by coupling the population update $\mb \beta^+$ with the finite-sample update $\tilde{\mb \beta}^+$. To this end, we make use of the fact that log-concave distributions are automatically sub-exponential (see Lemma \ref{lem:logconcave-subexponential}), so the random variables $\left\{X^{i}_{j} \tanh(0.5 F_{\mb \beta,\sigma}(\mb X^i))\right\}_{i=1}^{n}$ are i.i.d.\ sub-exponential for each coordinate $j$. Therefore, the concentration bound $\|\tilde{\mb \beta}^+-\mb \beta^+\|_2=\tilde{O} \big( \sqrt{(\|\mb \beta^*\|_2^2+\sigma^2) d/n} \big)$ holds, and we expect that the convergence properties of the population \vem carry over to the finite-sample case, modulo a statistical error of $\tilde{O}\big(\sqrt{d/n}\big)$. 

The above argument is made precise in following proposition for the one-dimensional case, which is proved in~Appendix \ref{subsec:prop-finite-sample-step}.
\begin{restatable}[1-d Finite Sample]{proposition}{onedfinite}
\label{prop:1d-finite-analysis}
Suppose the density function $f\in \mathcal{F}$ satisfies the regularity condition. With $\beta\in \R$ being the current estimate, the finite-sample \vem update~\eqref{eq:finite-vem-step} satisfies the following bound with probability at least $1-\delta$:
\begin{align}
\label{eq:finite-bound}
|\tilde{\beta}^+-\beta^{*}|
\leq\kappa(\beta^*,\beta,\sigma) \cdot |\beta-\beta^{*}|+(\beta^*+C_f\sigma)\cdot O\left(\sqrt{ \frac{1}{n}\log\frac{1}{\delta}}\right),
\end{align}
where $\kappa(\beta^*,\beta,\sigma)$ is contraction factor defined in Theorem \ref{thm:Global-Convergence-1D-VEM} and $C_f$ is the Orlicz $\Psi_1$  norm (i.e., the sub-exponential parameter) of a random variable with density $f\in \mathcal{F}$.
\end{restatable}

Using Proposition~\ref{prop:1d-finite-analysis}, we further deduce the global convergence of \vem in the finite sample case, which parallels the population result in Corollary~\ref{cor:tstep_convergence}. We develop this result assuming sample splitting, i.e., each iteration uses a fresh, independent set of samples. This assumption is standard in finite-sample analysis of EM~\cite{balakrishnan2017statistical,yan2017convergence,daskalakis2016ten,xu2016global,kwon2018global,klusowski2019estimating}. In this setting, we establish the following quantitative convergence guarantee for \vem initialized at any non-zero $\beta^0$. 

Without loss of generality, let $\beta^0,\beta^*>0$. The convergence has two stages. In the {\em first} stage, the \vem iterates  enter a local neighborhood around $\beta^*$, regardless of whether $\beta^0$ is close to or far from $0$. This is the content of the result below.

\begin{restatable}[First Stage: Escape from 0 and $\infty$]{proposition}{firststage}
\label{prop:first_stage}
Suppose the initial point $\beta^0$ is either close to $0$ (e.g, $\in(0,0.5\beta^*)$) or far away from $\beta^*$  (e.g, $\in(1.5\beta^*,\infty)$). After $T = O\left({\log  \frac{\frac{0.25\beta^*}{|\beta^0-\beta^*|}}{\log \kappa(\beta^*,\min(\beta^0,0.5\beta^*),\sigma)}}\right)$ iterations, with 
$N/T =\Omega\left(\frac{(1+C_f/\eta)^2}{(1-\kappa(\beta^*,\min(\beta^0,0.5\beta^*),\sigma))^2}\log \frac{1}{\delta}\right)$ fresh samples per iteration, 
\vem outputs a solution $\tilde{\beta}^T\in (0.5\beta^*,1.5\beta^*)$ with probability at least $1-\delta\cdot O\left(\log  {\frac{0.25\beta^*}{|\beta^0-\beta^*|}}/{\log \kappa(\beta^*,\min(\beta^0,0.5\beta^*),\sigma)}\right) $.   
\end{restatable}

Within this local neighborhood, the \vem iterates converge to $\beta^*$ geometrically, up to a statistical error determined by the sample size. This {\em second} stage convergence result is given below.

\begin{restatable}[Second Stage: Local Convergence]{proposition}{secondstage}
\label{prop:second_stage}
The following holds for any $\epsilon>0$.
Suppose $\beta^0 \in (0.5\beta^*,1.5\beta^*)$. After $T = O\left({\log \epsilon}/{\log \kappa(\beta^*,0.5\beta^*,\sigma)}\right)$ iterations, with $N/T = \Omega(\frac{(\beta^*+C_f/\eta)^2}{\epsilon^2(1-\kappa(\beta^*,0.5\beta^*,\sigma))^2}\log \frac{1}{\delta}) $ fresh samples per iteration, \vem outputs a solution $\tilde{\beta}^T$ satisfying $|\tilde{\beta}^T - \beta^*|\leq \epsilon \beta^*$ with probability at least $1-\delta \cdot O\left({\log \epsilon}/{\log \kappa(\beta^*,0.5\beta^*,\sigma)}\right)$.
\end{restatable}

We prove Propositions~\ref{prop:first_stage} and~\ref{prop:second_stage} in Appendix \ref{subsec:global_finite}.

Next, we parse the above results in the special cases of Gaussian, Laplace and Logistic, assuming that $ \sigma = 1$ for simplicity. Accordinly, $\eta = \beta^*$.  In Section~\ref{sec:analysis-d=1} we showed that $\kappa(\beta^*,\beta,\sigma) = \exp\left(-\min(\beta,\beta^*)^{\gamma}\right)$, where $\gamma \equiv \gamma_f$ is the growth rate of the log density $-\log f$. Consequently, the first stage  requires $ O\left({1}/{(\min(\beta^0,\beta^*))^{\gamma}}\right)$ iterations with $\tilde{\Omega}\left({1}/{(\min(\beta^0,\beta^*))^{2\gamma}}\right)$ samples per iteration, and the second stage requires $ O\left({\log (1/\epsilon)}/{\eta^\gamma} \right)$ iterations with  $\tilde{\Omega}\left({1}/{\epsilon^2 \eta^{2\gamma} }\right)$ samples per iteration. It is seen that we have better iteration and sample complexities with a larger $\eta \ge 1$ (larger separation between the components) and a larger $\gamma$ (lighter tail of the components).

In contrast, in the low SNR regime with $\eta<1$, the sample complexity actually becomes worse for a larger $\gamma$ (lighter tails). Indeed, low SNR means that two components are close in location when $\sigma=1$. If their tails are lighter, then it becomes more likely that the mixture density $(f_{\beta^*,\sigma}+f_{-\beta^*,\sigma})/2$ has a unique mode at 0 instead of two modes at $\pm \beta^*$. In this case, the mixture problem becomes harder as it is more difficult to distinguish between the two components.

In the higher dimensional setting, we can similarly show  coupling in $\ell_2$ (i.e., bounding $\|\tilde{\mb \beta}^+-\mb \beta^+\|_2$) via sub-exponential concentration. However, extending the convergence results above to $d>1$ is more subtle, due to the issue of $\ell_2$ non-contraction (see Lemma~\ref{lem: non-contraction}). Addressing this issue would require coupling in a different metric (e.g., in angle---see~\cite{kwon2018global,xu2016global}); we leave this to future work.

\section{Robustness Under Model Mis-specification}
\label{sec:misspecification}
In practice, it is sometimes difficult to know a priori the exact parametric form of a log-concave distribution that generates the data. This motivates us to consider the following scenario: the data is from the mixture $D(\mb \beta^*, \sigma) $ in~\eqref{eq:generative_process} with a true log-concave distribution $f \in \mathcal{F}$ and unknown location parameter~$\mb \beta^*$, but we run \vem assuming some other log-concave distribution $\hat{f}(\cdot)  = C_{\hat{g}}^{-1} \exp(-\hat{g}(\|\cdot\|_2)) \in \mathcal{F}$. Using the same symmetry argument as in deriving~\eqref{eq:one_comp_variant_update}, we obtain the following expression for the mis-specified \vem update in the population case:
\begin{align}
\label{eq:missepcified_vem_update}
\hat{\mb \beta}^+ 
= \hat{M}(\mb \beta^*, \mb\beta)
:= 
\bb{E}_{\mb X\sim f_{\mb\beta^*,\sigma}} \mb X\tanh \big(0.5 \hat{F}_{\mb\beta,\sigma}(\mb X) \big),
\end{align}
where $\hat{F}_{\mb\beta,\sigma}(\mb X) := \hat{g}\left(\frac{1}{\sigma}\|\mb X+\mb\beta\|_2\right) - \hat{g}\left(\frac{1}{\sigma}\|\mb X-\mb\beta\|_2\right)$.

Multiple properties of the \vem update are preserved in the mis-specification setting. In particular, using the same approach as in Lemma \ref{lem:two-dim-structure} and Lemma \ref{lem:angle-decrease}, we can show that the mis-specified \vem update is also a two dimensional object and satisfies the same strict angle decreasing property $\angle (\hat{\mb \beta}^+, \mb \beta^*) < \angle (\mb \beta, \mb \beta^*)$. Therefore, to study the convergence behavior of mis-specified \vem, it suffices to understand the one-dimensional case (i.e., along the $\mb\beta^*$ direction). 

We provide results focusing on the setting in which $\hat{f}$ is Gaussian, that is, we fit a Gaussian mixture to a true mixture of log concave distributions. In this setting, we can show that mis-specified \vem has only 3 fixed points $\{\pm \bar{\beta}, 0\}$ (Lemma~\ref{lem:mis-3-fixed}). Moreover, we can bound the distance between $\bar{\beta}$ and the true $\beta^*$, thereby establishing the following convergence result:

\begin{restatable}[Fit with 2GMM]{proposition}{misspecifiedgmm}
\label{lem:mis-gaussian}
Under the above one dimensional setting with Gaussian $\hat{f}$, the following holds for some  absolute constant $C_0>0$: 
If $ \eta \ge C_0$, then the \vem algorithm with a non-zero initialization point $\beta^0$ converges to a solution $\bar{\beta}$ satisfying $\textup{sign}(\bar{\beta})= \textup{sign}(\beta^0)$ and
\begin{align*}
     \big| \bar{\beta} - \textup{sign}(\beta^0\beta^*)\beta^* \big|  \le 10 \sigma.
\end{align*}
\end{restatable}
We prove this proposition in Appendix \ref{subsec:proof-mis-gaussian}. The proposition establishes the robustness of \vem: even in the mis-specified setting, \vem still converges globally. Moreover, when the SNR $\eta$ is high (i.e., small noise level $\sigma$), the final estimation error is small and scales linearly with $\sigma$.

\section{Conclusion}
\label{sec:conclusion}
In this paper, we have established the global convergence of the \vemlong algorithm for a mixture of two log-concave densities. The rotation invariance property is the only requirement for a theoretical guarantee. An immediate future direction is to establish quantitative global convergence guarantees in high dimensions for both population and finite sample case, which would require generalizing the angle convergence property in \cite{kwon2018global} to log-concave distributions. It is also of interest to relax the rotation invariance assumption (as many interesting log-concave distributions are skewed) and to consider mixtures with multiple components.

\section*{Acknowledgement}
W. Qian and Y. Chen are partially supported by NSF CRII award 1657420 and grant 1704828. Y. Zhang is supported by NSF award 1740822.
\bibliographystyle{plain}  
\bibliography{variant_em}


\appendixpage
\appendix


\section{Additional Notations for Appendix}
We use $\hat{\mb\beta}$ to denote the unit vector of $\mb\beta$, and $\mb\beta^{\perp}$ to denote a vector orthogonal to $\mb\beta$.  $\mb e_i\in\R^d$ is the $i$-th standard basis vector.

\section{Elementary Properties of Log-concave Distributions}
\label{appendix-log-concave}
A function $f:\bb{R}^d\to \bb{R}$ is \textit{log-concave} if it satisfies:
\begin{align*}
f(\alpha(\mb x)+(1-\alpha)\mb y)\geq f(\mb x)^\alpha f(\mb y)^{1-\alpha},
\end{align*}
for every $\mb x,\mb y\in \bb{R}^d$ and $0\leq \alpha \leq 1$. Equivalently, $\log f$ is a concave function. We consider log-concave distribution $f$ which further satisfies: $\int_{\bb{R}^d} f(\mb x)d\mb x=1$. The following is a classical result for log-concave distributions, which says that the log-concavity property is preserved by marginalization and convolution. 

\begin{theorem}
\label{thm:marginalization_convolution}
All marginals as well as the density function of a log-concave distribution is log-concave. The convolution of two log-concave distributions is again a log-concave distribution. 
\end{theorem} 

 The log-concave distribution on $\bb{R}$ has the following monotone likelihood ratio property:

\begin{proposition}
A density function $f$ on $\bb{R}$ is log-concave if an only if the translation family $\left\{f(\cdot - \theta): \theta \in \bb{R}\right\}$ has a monotone likelihood ratio: for every $\theta_1<\theta_2$, the ratio $\frac{p(x-\theta_2)}{p(x-\theta_1)}$ is a monotone nondecreasing function of $x$.   
\end{proposition}
Furthermore, log-concave  distribution has finite moments of all order. 
\begin{lemma}
\label{lem:bounded_moments}
For a rotation invariant log-concave density: $\bb{R}^d\to \bb{R}$, all the moments exist. 
\end{lemma}
\begin{proof}
It suffices to show that $\int_{\mb x} |x_1|^k f(\mb x) \dup {\mb x} <\infty$. By the rotation invariant property, we need to show:
\begin{align}
\int |x_1|^k \left( \int f(\mb x) \dup x_2\ldots \dup x_d \right) \dup x_1<\infty.
\end{align}
Note that $f_{2:d}(x_1):=\int_{x_2,\ldots x_d} f(\mb x) \dup x_2\ldots \dup x_d = \exp (-g_{2:d}(x_1))$ is the marginal distribution, thus log-concave by Theorem \ref{thm:marginalization_convolution}. The problem is now further  reduced to show that a one-dimensional symmetric log-concave distribution $f(x)=\exp(-g(x))$ has finite moments. By the convexity, 
\begin{align}
g(x) \geq g(x_0)+\partial g(x_0) (x-x_0),
\end{align}
for some $x_0>0$ and $g^{\prime}(x_0)>0$. In particular, we have that shown that there exist $x_0,a,b>0$ such that $g(x) \geq b+a(x-x_0)$. Therefore,
\begin{align*}
& \int_{x} |x|^k \exp(-g(x)) \dup x \\
= & 2\int_{x\geq 0} x^k \exp(-g(x)) \dup x \\
\leq & 2\int_{x\geq 0} x^k \exp(-b-a(x-x_0)) \dup x \\
= & 2\exp(-b+ax_0)\int_{x \geq 0} x^k \exp(-ax) \dup x <\infty.
\end{align*}
We conclude that all the moments exist.
\end{proof}
We refer the reader to \cite{saumard2014log} and \cite{walther2002detecting} for a detailed review for other properties of log-concave distributions. 

\section{Analysis for $d=1$}
\label{appendix-subsec:d=1}
In this section, we prove the convergence results for $d=1$. Especially, the proof of Theorem \ref{thm:Global-Convergence-1D-VEM} is presented in Section \ref{subsec:proof-thm-Global-Convergence-1D-VEM}, and in Section \ref{subsec:concrete-convergence-rate}, we discuss the convergence rate for some explicit log-concave distribution examples.

\subsection{Proof of Theorem \ref{thm:Global-Convergence-1D-VEM}}
\label{subsec:proof-thm-Global-Convergence-1D-VEM}

We recall the shorthand notation:
\[
F_{\beta,\sigma}(x) = g\left(\frac{1}{\sigma}|x+\beta|\right) -  g\left(\frac{1}{\sigma}|x-\beta|\right).
\]
When $\sigma=1$, we abbreviate $F_{\beta,\sigma}$ as $F_{\beta}$. 
For readability, we restate the theorem here:

\onedglobal*

\begin{proof}
Without loss of generality, $\beta^*>0$.
When $\beta=0$, one verifies that $M(\beta^{*},0)=0$, therefore $0$ is a trivial fixed point. 
Without loss of generality, we assume $\beta>0,\beta^{*}>0$ (in the case for $\beta<0$, replace $\beta^*$ with $-\beta^*$.) By scaling $\beta^* \to \frac{\beta^*}{\sigma}$, $\beta\to \frac{\beta}{\sigma}$ and $M(\beta^*,\beta)\to \frac{M(\beta^*,\beta)}{\sigma}$, we can further assume that $\sigma=1$ in the following analysis.  

We first establish the consistency property of the $\vem$ update: $M(\beta,\beta)=\beta$ for all $\beta$. This follows from the algebra: 
\begin{align*}
M(\beta,\beta) & =\int_x\frac{1}{2}(f(x-\beta)+f(x+\beta))x\left[\frac{f(x-\beta)-f(x+\beta)}{f(x-\beta)+f(x+\beta)}\right]\dup x\\
 & =\frac{1}{2}\int_x x(f(x-\beta)-f(x+\beta))\dup x\\
 & =\frac{1}{2}\int_x(x-\beta)f(x-\beta)\dup x-\frac{1}{2}\int(x+\beta)f(x+\beta)\dup x+\beta\\
 & =\beta,
\end{align*}
where the last step holds since $x\to xf(x)$ is an odd function. Consequently, the
integral $\int_x(x-\beta)f(x-\beta)\dup x$ and $\int_x(x+\beta)f(x+\beta)\dup x$
vanish. 

We next argue that the \vem update has a unique fixed point $\beta^{*}$ when $\beta>0$ ($-\beta^*$ is another fixed point when $\beta<0$ by symmetry). In the region where $\beta>\beta^{*}$, we have
\begin{align}
M(\beta^{*},\beta)-\beta^{*} & =M(\beta^{*},\beta)-M(\beta,\beta)+\beta-\beta^{*}\nonumber \\
 & =\frac{\partial M(z,\beta)}{\partial z}\mid_{z\in(\beta^{*},\beta)}(\beta^{*}-\beta)+\beta-\beta^{*}\nonumber \\
 & =(\beta-\beta^{*})\left(1-\frac{\partial M(\beta,z)}{\partial z}\mid_{z\in(\beta^{*},\beta)}\right)\nonumber \\
 & \leq\sup_{z\in(\beta^{*},\beta)}\left(1-\frac{\partial M(\beta,z)}{\partial z}\right)(\beta-\beta^{*}).\label{eq:big_contract-1}
\end{align}
In the first step above, we decompose the difference using the consistency property. The allows us to apply the intermediate value theorem for function $M(\cdot,\beta)$ with respect to the first argument in the second step above.  
In the case when $\beta<\beta^{*}$, we can derive the following relation in a similar way:
\begin{align}
\beta^{*}-M(\beta,\beta^{*}) & =\beta^{*}-\beta+M(\beta,\beta)-M(\beta,\beta^{*})\nonumber \\
 & =\beta^{*}-\beta+\frac{\partial M(\beta,z)}{\partial z}\mid_{z\in(\beta,\beta^{*})}(\beta-\beta^{*})\nonumber \\
 & =(\beta^{*}-\beta)\left(1-\frac{\partial M(\beta,z)}{\partial z}\mid_{z\in(\beta,\beta^{*})}\right)\nonumber \\
 & \leq\sup_{z\in(\beta,\beta^{*})}\left(1-\frac{\partial M(\beta,z)}{\partial z}\right)(\beta^{*}-\beta).\label{eq:small_contract-1}
\end{align}
In view of the above two cases, we conclude that: If $\beta\beta^{*}>0$,
\begin{align}
\label{eq: bound-pop-update}
|M(\beta^{*},\beta)-\beta^{*}|\leq\underbrace{\sup_{t\in[0,1]}\left[1-\frac{\partial M(z,\beta)}{\partial z}\mid_{z=t\beta^{*}+(1-t)\beta}\right]}_{\kappa(\beta^{*},\beta)}|\beta-\beta^{*}|.
\end{align}
The problem is reduced to lower bound  $\frac{\partial M (z,\beta)}{\partial z}$, where $z$ is between $\beta$ and $\beta^*$. Recall:
\begin{align*}
M(z,\beta)= & \mathbb{E}_{X\sim f_z}X\tanh\left(0.5 F_{\beta}(X)\right)\\
= & \int_{x}f(x-z)(x)\tanh\left(0.5 F_{\beta}(x)\right)\dup x \\
= & \int_{x}\underbrace{f(x)(x+z)\tanh\left(0.5 F_{\beta}(x+z)\right)}_{h(x,z)}\dup x.
\end{align*}
In the last step, we applied change of variable for the term $x-z$. To differentiate $M$ with respect to $z$, we can interchange the order of differentiation and integral: $\frac{\partial}{\partial z} \int_x h(x,z)dx = \int_x \frac{\partial }{\partial z} h(x,z) dx$ by the regularity condition. Note that $\frac{\partial h(x,z)}{\partial z}$ has the following expression:
\begin{align*}
\frac{\partial h(x,z)}{\partial z}= & f(x)\Bigg(\tanh\left(0.5 F_{\beta}(x+z)\right) 
 +0.5(x+z)\left(\frac{\partial}{\partial x} F_{\beta}(x+z)\right)\tanh^{\prime}\left(0.5 F_{\beta}(x+z)\right)\Bigg). 
\end{align*}
Therefore,
\begin{align*}
\frac{\partial M(z,\beta)}{\partial z}= & \underbrace{\mathbb{E}_{X\sim f_z}\tanh\left(0.5F_{\beta}(X)\right)}_{T_{1}}
 +\underbrace{\mathbb{E}_{X\sim f_z}\left[0.5XF_{\beta}^{\prime}(X)\tanh^{\prime}\left(0.5F_{\beta}(X)\right)\right]}_{T_{2}}.
\end{align*}
From Lemma \ref{lem:derivative-opt-param-1}, we see that $T_1,T_2\geq 0$, thus a lower bound for $\frac{\partial M(z,\beta)}{\partial z}$ follows:
\begin{align}
\frac{\partial M(z,\beta)}{\partial z} \geq &  \mathbb{E}_{X\sim f_z}\tanh\left(0.5 F_{\beta}(X)\right) \nonumber \\
\geq & \mathbb{E}_{X\sim f_{\min(\beta,\beta^*)}}\tanh\left(0.5 F_{\min(\beta,\beta^*)}(X)\right) \label{eq:increasing},
\end{align}
where \eqref{eq:increasing} holds since $\mathbb{E}_{X\sim f_z}\tanh\left(0.5 F_{\beta}(X)\right)$ increases with $z$ and $\beta$, which is also established in Lemma \ref{lem:derivative-opt-param-1}.  

Combining inequalities \eqref{eq: bound-pop-update} and \eqref{eq:increasing}, we conclude that
\begin{align}
\label{eq:1d-bound}
|M(\beta^*,\beta)-\beta^*|\leq \mathbb{E}_{X\sim f_{\min(\beta,\beta^*)}}[1-\tanh(0.5F_{\min(\beta,\beta^*)}(X))]|\beta - \beta^*|.
\end{align}
 $\kappa (\beta^*,\beta,\sigma )\in (0,1)$ by Corollary \ref{lem:rate<1}. From the bound in \eqref{eq:1d-bound}, we see that $M(\beta^*,\beta,\sigma)$ moves closer to $\beta^*$ whenever $\beta >0$ and $\beta \neq \beta^*$, therefore, $\beta^*$ is the unique fixed point on $\beta>0$. Similarly, $-\beta^*$ is the unique fixed point on $\beta<0$.
We have completed the proof of Theorem \ref{thm:Global-Convergence-1D-VEM}.
\end{proof}

\subsubsection{Supporting Lemmas of Theorem \ref{thm:Global-Convergence-1D-VEM}}

\begin{lemma}[Positive Integral] 
	\label{lem:positive_integral} Let $f\in \mathcal{F}$. Let $S\subseteq \R^{+}$ be a set with non-zero measure. Suppose that $h$ is an odd function with $h(x)\geq 0$ on $[0,\infty]$ and $h(x)>0$ on $S$. If $z>0$, the following holds:
	\begin{align}
	\mathbb{E}_{X\sim f_z} h(X)>0.
	\end{align}
\end{lemma}
\begin{proof}
	We have
	\begin{align}
	\R_{X\sim f_z} h(X) =&  \int_x f(x-z)h(x) \dup x \nonumber\\
	= & \int_{x\geq 0}f(x-z)h(x) \dup x + \int_{x\leq 0} f(x-z)h(x) \dup x \nonumber \\
	= &  \int_{x\geq 0}f(x-z)h(x) \dup x + \int_{x\geq 0} f(x+z)(-h(x)) \dup x \label{eq:odd_property} \\
	= & \int_{x\geq 0} (f(x-z)-f(x+z)) h(x) \dup x \nonumber \nonumber \\
	= & \int_{x\in S} (f(x-z)-f(x+z)) h(x) \dup x, \nonumber
	\end{align}
	where the equality \eqref{eq:odd_property} holds since $h$ is an odd function. Since $f\in \mathcal{F}$, $f(x-z)-f(x+z)> 0$ on $x\geq 0$. We thus conclude that the above integral is positive. 
\end{proof}

\begin{lemma}
\label{lem:derivative-opt-param-1}Suppose the density function $f$ satisfies the regularity condition,
\begin{align*}
\frac{\partial M (z,\beta)}{\partial z}= & \underbrace{\mathbb{E}_{X\sim f_z}\tanh\left(0.5F_{\beta}(X)\right)}_{T_{1}}
 +\underbrace{\mathbb{E}_{X\sim f_z}\left[0.5X F_{\beta}^{\prime}(X)\tanh^{\prime}\left(0.5F_{\beta}(X)\right)\right]}_{T_{2}}.
\end{align*}
Moreover, $T_1$ and $T_2$ have following properties:
\begin{itemize}
    \item If $\beta z>0$, then $T_{1}>0$ and $T_{2}>0$;
    \item $T_{1}$
is an increasing function with respect to both $\beta$ and $z$.
\end{itemize}
\end{lemma}
\begin{proof}
Without loss of generality, we assume $z>0$ and $\beta>0$. Note that the integrand for both $T_1$ and $T_2$ are odd in $x$. Moreover, they are both strictly positive on a subset of $\left\{x:x\geq 0\right\}$ with a positive measure when $\beta>0$ explained as follows:
\begin{itemize}
    \item For the intergrand  for $T_1$, $\tanh(0.5F_{\beta}(x))$, we have that $F_{\beta}(x) = g(|x+\beta|)-g(|x-\beta|)$. When $x>\beta$, $g(x+\beta)>g(x-\beta)$ since $g$ is increasing on $(0,\infty)$. When $x\in (0,\beta)$, $g(x+\beta)>g(\beta-x)=g(|x-\beta|)$ by the monotonicity of $g$ on $[0,\infty]$. Combining two cases, we see that the integrand for $T_1$ is strictly positive when $x>0$.  
    \item For the integrand for $T_2$, $x F_{\beta}^{\prime}(x)\tanh^{\prime}\left(0.5F_{\beta}(x)\right)$, we have that $x\tanh^{\prime}\left(0.5F_{\beta}(x)\right)>0$ when $x>0$. For $F_{\beta}^{\prime}(x) = g^{\prime}(x+\beta)-g^{\prime}(x-\beta)$, we first note that since $g$ is convex, $g^{\prime}(x+\beta)-g^{\prime}(x-\beta)\geq 0$ for all $x\geq 0$. On $x\in (0,\beta)$, $F_{\beta}(x)=g(x+\beta)-g(\beta-x)$ has the property that $F_{\beta}(0)=0$ and $F_{\beta}(x)>0$ is increasing for all $x\in (0,\beta)$ by the monotonicity of $g$ on $[0,\infty]$, it follows that $F_\beta^{\prime}(\cdot)$ is positive on $(0,\beta)$. ($g(x)=|x|$ is a typical example, where $F_{\beta}(\cdot)$ only has positive differential in $(0,\beta)$ when the domain is restricted to $\mathbb{R}^+$.) 
\end{itemize}
Therefore both $T_1$ and $T_2$ are positive by Lemma \ref{lem:positive_integral}. To see how $T_1$ changes with respect to $\beta$, let us take the derivative with respect to $\beta$ (again, regularity condition allows us to change the order between differentiation and integration):
\begin{align*}
    \frac{\partial T_1}{\partial \beta} = \mathbb{E}_{X\sim f_z} 0.5 (g^\prime(X+\beta)+g^\prime(X-\beta))\tanh^{\prime}(0.5F_{\beta}(X)).
\end{align*}
The integrand is odd in $x$. On $x\in (0,\beta)$, $g^\prime(x+\beta)+g^\prime(x-\beta) = g^\prime(x+\beta)-g^\prime(\beta-x) \geq 0$ by the convexity of $g$. On $x\in (\beta,\infty)$, $g^\prime(x+\beta)+g^\prime(x-\beta)> 0$ as $g$ is strictly increasing on $[0,\infty]$. Therefore the integrand is positive on a subset of $[0,\infty]$ with non-zero measure. We conclude that $\frac{\partial T_1}{\partial \beta}> 0$ by Lemma \ref{lem:positive_integral} and $T_1$ is an increasing function with respect to $\beta$.  

To see how $T_1$ changes with respect to $z$, let us take the derivative with respect to $z$:
\begin{align*}
    \frac{\partial T_1}{\partial z} = \mathbb{E}_{X\sim f_z} 0.5(g^\prime(X+\beta)-g^\prime(X-\beta))\tanh^{\prime}(0.5F_{\beta}(X)).
\end{align*}
Similar to the integrand of $T_2$, we can show that the integrand for $\frac{\partial T_1}{\partial z}$ is positive on a subset of $[0,\infty]$ with a positive measure when $\beta>0$. Therefore, $\frac{\partial T_1}{\partial z}>0$ and $T_1$ is increasing with respect to $z$ by Lemma \ref{lem:positive_integral} . 
\end{proof}

\begin{corollary}
\label{lem:rate<1}
If $z>0$, then $0<\mathbb{E}_{x\sim z}(1-\tanh(\frac{g(x+z)-g(x-z)}{2}))<1.$
\end{corollary}
\begin{proof}
Since $|\tanh(\cdot)|<1$, we have
\[
\mathbb{E}_{X\sim f_z} \tanh(F_{z}(X))<1,
\]
thus proving the upper bound in the corollary. For the lower bound, we use the previous argument that $F_z(x)>0$ and $F_z(\cdot)$ is odd, and apply Lemma~\ref{lem:positive_integral} to obtain that 
\[
\mathbb{E}_{X\sim f_z} \tanh(F_{z}(X))>0
\]
The proof of the corollary is completed.. 
\end{proof}

\subsection{Convergence Rates for Specific Distributions}
\label{subsec:concrete-convergence-rate}

In this section, we derive the explicit convergence rates for specific log-concave  distributions. Without loss of generality, we may assume that $\sigma=1$ in the derivation. Replace $\beta$ by $\frac{\beta}{\sigma}$ gives the general result. 

\subsubsection{Gaussian}
In this case, we have $f(x)=\frac{1}{\sqrt{2\pi}}\exp \left(-\frac{x^{2}}{2}\right),$
and $\tanh(0.5F_{\beta}(x)) =\tanh \left(\beta x\right)$. We compute $\kappa(\beta,\beta^*,\sigma)$ as follows: for any $\beta>0$, there holds the bound
\begin{align*}
 & \mathbb{E}_{X\sim f_{\beta}}\left[1-\tanh \left(\beta X \right)\right]\\
= & \mathbb{E}_{X\sim f_{\beta}}\frac{2\exp(-\beta X)}{\exp(-\beta X)+\exp(\beta X)}\\
\leq & \mathbb{E}_{X\sim f_{\beta}}\exp \left(-\beta X\right)=\exp \left(-\frac{\beta^{2}}{2}\right).
\end{align*}
We thus obtain the following corollary.
\begin{corollary}[Gaussian]
\label{cor:gaussian}
Suppose $f\in \mathcal{F}$ is Gaussian, $f(x) =\frac{1}{\sqrt{2\pi}}\exp \left(-\frac{x^{2}}{2}\right)$,  we have
$\kappa(\beta,\beta^*,\sigma) =\exp \left(-\frac{\min(\beta,\beta^{*})^{2}}{2\sigma^{2}}\right) $.
\end{corollary}

\subsubsection{Laplace}
In this case, we have $f(x)=\frac{1}{2}\exp(-|x|)$, and $\tanh(0.5 F_{\beta}(x))=\frac{\exp(-|x-\beta|)-\exp(-|x+\beta|)}{\exp(-|x-\beta|)+\exp(-|x+\beta|)}$.
We thus have
\begin{align}
 & \mathbb{E}_{X\sim f_{\beta}}\left[1-\tanh(0.5F_{\beta}(X))\right] \nonumber \\
= & \int_{x}\exp(-|x-\beta|)\frac{\exp(-|x+\beta|)}{\exp(-|x-\beta|)+\exp(-|x+\beta|)} \dup x \nonumber\\
= & \int_{x\leq-\beta}\exp(x-\beta)\frac{\exp(x+\beta)}{\exp(x-\beta)+\exp(x+\beta)}\dup x \nonumber\\
&  +\int_{x\geq\beta}\exp(-x+\beta)\frac{\exp(-x-\beta)}{\exp(-x+\beta)+\exp(-x-\beta)}\dup x \nonumber \\
 & +\int_{x\in(-\beta,\beta)}\exp(x-\beta)\frac{\exp(-x-\beta)}{\exp(x-\beta)+\exp(-x-\beta)}\dup x \nonumber \\
= & 2\frac{\exp(-\beta)}{\exp(-\beta)+\exp(+\beta)}+\exp(-\beta)\int_{x=-\beta}^{\beta}\frac{1}{\exp(x)+\exp(-x)}\dup x \nonumber \\
\leq  & 2\frac{\exp(-\beta)}{\exp(-\beta)+\exp(+\beta)}+2\frac{1}{1+\exp(-2\beta)}\exp(-\beta)(1-\exp(-\beta)) \label{eq:step1} \\
= & 2\frac{\exp(-\beta)}{1+\exp(-2\beta)}, \nonumber 
\end{align}
where in \eqref{eq:step1} we used the elementary inequality:
\[
\frac{1}{\exp(x)+\exp(-x)}\leq\frac{1}{1+\exp(-2\beta)}\exp(-x),
\qquad\forall x\in(0,\beta).
\]
We thus obtain the following corollary.
\begin{corollary}[Laplace]
\label{cor:laplace}
Suppose that $f\in \mathcal{F}$ is Laplace, $f(x) = \frac{1}{2\sigma_0}\exp(-|x/\sigma_0|)$, with $\sigma_0=\frac{1}{\sqrt{2}}$. We have
$\kappa(\beta,\beta^*,\sigma) = \frac{2\exp(-\frac{1}{\sigma_0\sigma}\min(\beta,\beta^{*}))}{1+\exp(-2 \frac{1}{\sigma_0\sigma}\min(\beta,\beta^{*}))} $.
\end{corollary}

\subsubsection{Logistic}
In this case, we have $f(x)=\frac{\exp(x)}{(1+\exp(x))^{2}}$, and
$\tanh(0.5F_{\beta}(x))=\frac{(\exp(\beta)-\exp(-\beta))(\exp(2x)-1)}{(\exp(\beta)+\exp(-\beta))(1+\exp(2x))+4\exp(x)}$. 
Therefore, there holds the bound
\begin{align*}
&\mathbb{E}_{X\sim f_{\beta}}\left[1-\tanh(0.5F_{\beta}(X))\right]  \\
= & 2\int\frac{\frac{\exp(x-\beta)}{(1+\exp(x-\beta))^{2}}\frac{\exp(x+\beta)}{(1+\exp(x+\beta))^{2}}}{\frac{\exp(x-\beta)}{(1+\exp(x-\beta))^{2}}+\frac{\exp(x+\beta)}{(1+\exp(x+\beta))^{2}}}\dup x\\
= & 2\int\frac{\exp(x)}{\exp(-\beta)(1+\exp(x+\beta))^{2}+\exp(\beta)(1+\exp(x-\beta))^{2}}\dup x\\
= & 2\int\frac{\exp(x)}{\exp(-\beta)+\exp(\beta)+4\exp(x)+\exp(2x)(\exp(\beta)+\exp(-\beta))}\dup x\\
= & 2\int\frac{1}{(\exp(-x)+\exp(x))(\exp(-\beta)+\exp(\beta))+4}\dup x\\
= & 4\int_{x\geq0}\frac{1}{(1+\exp(2x))(\exp(-\beta)+\exp(\beta))+4\exp(x)} \dup (\exp(x))\\
= & 4\int_{s=1}^{\infty}\frac{1}{(1+s^{2})c+4s} \dup s,
\end{align*}
where $c:=\exp(-\beta)+\exp(\beta)$. We can further upper bound the
last integral, yielding the bound
\begin{align*}
&\mathbb{E}_{X\sim f_{\beta}}\left[1-\tanh(0.5F_{\beta}(X))\right]  \\
= & 4\int_{s=1}^{\infty}\frac{1}{(\sqrt{c}s+\frac{2}{\sqrt{c}})^{2}+c-\frac{4}{c}}ds\\
\leq & 4\frac{1}{\sqrt{c}} \left[-\frac{1}{x}\right]_{x=\sqrt{c}+\frac{2}{\sqrt{c}}}^{\infty}\\
= & \frac{4}{\sqrt{c}}\frac{1}{\sqrt{c}+\frac{2}{\sqrt{c}}}
=\frac{4}{c+2}
=\frac{4}{\exp(\beta)+\exp(-\beta)+2}<1,
\end{align*}
where in the second step, we use the fact that $c\geq2$. 
We thus obtain the following corollary.
\begin{corollary}[Logistic]
\label{cor:logistic}
Suppose $f\in \mathcal{F}$ is Logistic, $f(x) =\frac{\exp(x/\sigma_0)}{\sigma_0(1+\exp(x/\sigma_0))^{2}}$, with $\sigma_0=\frac{\sqrt{3}}{\pi}$. We have
$\kappa(\beta,\beta^*,\sigma) = \frac{4}{\exp\left(\frac{\min(\beta,\beta^*)}{\sigma_0\sigma}\right)+\exp \left(-\frac{\min(\beta,\beta^*)}{\sigma_0\sigma}\right)+2}$.
\end{corollary}

\section{Analysis for $d>1$}
\label{appendix-subsec:d>1}
In this section, we prove the convergence result for the setting $d>1$. In Section \ref{subsec:proof-two-dim-structure}, we prove for Lemma~\ref{lem:two-dim-structure} that shows the population \vem update is two dimensional; in Section \ref{subsec:proof-high-d-theorem}, we present the proof for Theorem~\ref{thm:location-high-dim-asymptotic} on the asymptotic convergence to the true location parameter; in Section~\ref{subsec:non-contraction}, we prove Lemma \ref{lem: non-contraction} that demonstrates the non-contraction phenomenon of the \vem update within a general log-concave distribution family. 

\subsection{Proof of Lemma~\ref{lem:two-dim-structure}}
\label{subsec:proof-two-dim-structure}
For readability we restate the lemma below.

\twodimstructure*
\begin{proof}
Using the rotation invariance property, we adopt the following local orthonormal basis
$\{\mb v_{1},\ldots,\mb v_{d}\}$, with $\mb v_{1}=\hat{\mb\beta}$, and $\mb v_{2}=\hat{\mb\beta}^{\perp}$
satisfying span$(\mb v_{1},\mb v_{2})=$span$(\mb\beta,\mb\beta^{*})$ and $\langle \mb v_2,\mb\beta^*\rangle \geq 0$. Under this basis, $\mb\beta$ and $\mb\beta^{*}$ have non-zero entries only in the first two coordinates; that is,
\begin{align*}
\mb\beta  =\bigl( \|\mb\beta\|_2,0,\ldots,0 \bigr), 
\qquad \text{and} \qquad
\mb\beta^{*}  =\bigl( \underbrace{\langle\mb\beta^{*},\hat{\mb\beta}\rangle}_{\beta_{1}^{*}},\underbrace{\langle\mb\beta^{*},\hat{\mb\beta}^{\perp}}_{\beta_{2}^{*}}\rangle,0,\ldots,0 \bigr).
\end{align*}
Here we denote $\beta_1^*: = \|\mb\beta^*\|_2\cos(\angle (\mb\beta,\mb\beta^*))$ and $\beta_2^*: = \|\mb\beta^*\|_2 \sin(\angle (\mb\beta,\mb\beta^*))$. 
With this notation, for each $ \mb{x} \in \mathbb{R}^d $ we have
\begin{align*}
\mb x-\mb\beta =(x_{1}-\|\mb\beta\|_2,x_{2},\ldots,x_{d}), 
\qquad \text{and} \qquad
 \mb x-\mb\beta^{*}  =(x_{1}-\beta_{1}^{*},x_{2}-\beta_{2}^{*},x_{3},\ldots,x_{d}).
\end{align*}
From the expression \eqref{eq:one_comp_variant_update} for the \variant EM update:
\[
\mb\beta^+ = \bb{E}_{\mb X\sim f_{\mb\beta^*,\sigma}} \mb X \tanh\big(0.5F_{\mb\beta,\sigma}(\mb X)\big),
\]
the $j$-th coordinate of $\mb\beta^{+}$ is
\begin{align}
\beta_{j}^{+} & =\mathbb{E}_{\mb x\sim f_{\mb\beta^*,\sigma}}x_{j}\tanh\left(0.5F_{\mb\beta,\sigma}(\mb x)\right)\label{eq:high-dim-coordinate}\\
 & =\int_{x_{-j}}\int_{x_{j}}\underbrace{\frac{1}{\sigma^{d}C_{g}}\exp\left(-g \left(\big \|\frac{1}{\sigma}(\mb x-\mb\beta^{*})\big\|_2\right)\right)x_{j}\tanh\big(0.5F_{\mb\beta,\sigma}(\mb x)\big)}_{h_{j}} \dup x_{j} \dup x_{-j},\nonumber 
\end{align}
where $x_{-j}$ denotes all the coordinates that is not $x_{j}$.
It is easy to see $h_{j}$ is an odd function in $x_{j}$ when $j\geq3$,
therefore $\beta_{j}^{+}=0$ for all $j\geq3$ and the \variant M-step preserves the 2 dimensional structure.
Moreover, when $\beta_{1}^{*}=0$, i.e, $\mb\beta$ is in the orthogonal direction
to $\mb\beta^{*}$, we have $\beta_{2}^{+}=0$ as $h_{j}$ is an odd function
in $x_{1}$. When $\beta_{2}^{*}=0$, i.e, $\mb\beta$ is in the same direction
as $\mb\beta^{*}$, we have $\beta_{2}^{+}=0$ as $h_{j}$ is an odd function
in $x_{2}$. In other words, Span$(\mb\beta^*)$ and Span$(\mb\beta^{*\perp})$ are 1-dimensional invariant subspaces.
\end{proof}

\subsection{Proof of Theorem~\ref{thm:location-high-dim-asymptotic}} 
\label{subsec:proof-high-d-theorem}

For readability we restate the theorem below.

\highdimglobal*

\begin{proof}
Let $\mb\beta^{0}$ denote an initial point that is not in the orthogonal direction to $\mb\beta^{*}$. Without loss of generality, we assume $\langle\mb\beta^{0},\mb\beta^{*}\rangle>0$. There are two cases for $\mb \beta^0$, either $\mb \beta^{0}$ is in the span of $\mb\beta^{*}$ or $\mb \beta^{0}$ is not in the direction of $\mb\beta^{*}$. In the previous case, the iterates remain in the direction of $\mb\beta^{*}$ and converge to $\mb\beta^{*}$ from Lemma
\ref{lem:optimal-direction-convergence}. 

In the latter case, we argue that all the accumulation points (existence by the boundedness of the iterate) must be in the direction of $\mb\beta^*$. If there exists some $t>0$ such that $\angle (\mb\beta^t,\mb\beta^*) = 0$, we are reduced to the previous case where the iterates remain in the direction of $\mb \beta^*$. From now on, we assume that $\angle (\mb\beta^{t},\mb\beta^*)>0$ for all $t\geq 0$. 

Lemma \ref{lem:angle-decrease} establishes the crucial angle decreasing property of the variant EM update, which says that the angle between the iterates and $\mb\beta^*$ strictly decreases, i.e, $\angle (\mb\beta^{t+1},\mb\beta^*) < \angle (\mb\beta^{t},\mb\beta^*)$. Indeed $\left\{\angle (\mb\beta^t, \mb\beta^*), t=0, 1,\ldots \right\}$ is a monotonically decreasing sequence, thus this sequence converges to $\theta^\infty\geq 0$, with $\angle (\mb\beta^t,\mb\beta^*) \geq \theta^\infty$ for all $t$.   

If $\theta^\infty =0$, we are done. Otherwise, if $\theta^\infty >0$, let $\left\{\mb\beta^{n_k}\right\}$ be a subsequence converging to an accumulation point $\mb\beta^\infty$. We deduce that $\angle (\mb\beta^\infty,\mb\beta^*)= \theta^\infty>0$ since any subsequence of $\left\{\angle (\mb\beta^t,\mb\beta^*)\right\}$ converges to $\theta^\infty$. By the continuity of variant EM operator, the subsequence $\left\{M(\mb\beta^{n_k},\mb\beta^*)\right\}$ converges to $M(\mb\beta^\infty,\mb\beta^*)$. Note that (i) $\left\{M(\mb\beta^{n_k},\mb\beta^*)\right\} = \left\{\mb\beta^{n_{k}+1}\right\}$ and (ii) $\angle (M(\mb\beta^\infty,\mb\beta^*),\mb\beta^*) < \theta^\infty$. Thus, there must be some $k$ such that $\angle (\mb\beta^{n_{k}+1},\mb\beta^*)$ is strictly between $\angle (M(\mb\beta^\infty,\mb\beta^*),\mb\beta^*)$ and $\theta^\infty$, contradicting with the previous analysis that $\angle (\mb\beta^t,\mb\beta^*) \geq \theta^\infty$ for all $t>0$. 

Let $\mathcal{F}$ be the set of accumulation points of the iterates $\left\{\mb\beta^t\right\}$, which are all in the direction of $\mb\beta^*$. Since $\langle \mb\beta^{0},\mb\beta^{*}\rangle>0$, we can show that all the limit points have positive correlation with $\mb \beta^*$ (Lemma \ref{lem:along convegence}). If $\mathcal{F}$ contains some $\mb\beta^\infty\not \in \left\{\mb 0,\mb\beta^*\right\}$, we can show by the continuity of the \variant EM operator that there exists some $k$ such that $\mb\beta^k$ falls into a local neighborhood of $\mb\beta^*$ that does not include $\mb\beta^\infty$. Indeed, since $\mb \beta^\infty \in \mathcal{F}$, there exists $t>0$ such that $\mb \beta^t$ is close to $\mb \beta^{\infty}$. Apply the \variant EM operator to $\mb \beta^t$ for finitely many times produces an iterate $\mb \beta^k$ (in the sequence $\left\{\mb \beta^t\right\})$ that is close to $\mb \beta^*$ by Theorem \ref{thm:Global-Convergence-1D-VEM}. On the other hand, using a quantitative local bound (Proposition \ref{prop:local_quantitative}) for the iterates, we show that the iterates remains in a local region of $\mb\beta^*$ once it is inside in Corollary \ref{cor:non-escaping}. In particular, this implies that all the iterates after $\mb\beta^{k}$ are trapped in the local region of $\mb\beta^*$ that does not contain $\mb \beta^{\infty}$, contradicting with the assumption $\mb\beta^\infty$ is a limit point of the iterates. 

We conclude that $\mathcal{F}$ can be either (a)$\left\{\mb 0\right\}$, (b)$\left\{\mb\beta^{*}\right\}$ or (c)$\left\{\mb 0,\mb\beta^*\right\}$. Note that  (c) is not possible for a similar reasoning as before (as one iterate is guaranteed to be in a local region of $\mb\beta^*$). (a) is not possible either. (a) implies that $\lim_t \mb \beta^t \to \mb 0$. In particular, there exist $N$ such that for all $n>N$, $\|\mb\beta^n\|_2\leq \frac{1}{8}\|\mb\beta^*\|_2$. In Lemma \ref{lem:norm_increasing}, we show that if all the iterates after $\mb\beta^{n}$ are non-zero and have norm no greater than $\frac{1}{8}\|\mb\beta^*\|_2$, the norm of the iterates must be lower bounded in the limit. Thus they can not converge to $\mb 0$. (b) is the only possibility and we are done.
\end{proof}

\subsubsection{Supporting Lemmas for Theorem~\ref{thm:location-high-dim-asymptotic} }

Below we record several technical lemmas used in the proof of Theorem~\ref{thm:location-high-dim-asymptotic}. Recall the shorthand notation:
\[
F_{\mb\beta,\sigma}(\mb x)=g\left(\frac{1}{\sigma}\|\left(x_{1}+\|\mb\beta\|_2,x_{2},\ldots,x_{d}\right)\|_2\right)-g\left(\frac{1}{\sigma}\|\left(x_{1}-\|\mb\beta\|_2,x_{2},\ldots,x_{d}\right)\|_2\right).
\]
$F_{\mb \beta,\sigma}(\mb x)$ is an odd function in $x_1$, and it is an even function in $x_2,\ldots,x_d$.
\begin{lemma}[Angle Decreasing]
\label{lem:angle-decrease}Suppose that the density function $f$ satisfies the regularity condition. $\beta_{2}^{+}>0$ whenever
$\beta_{2}^{*}>0$ or $\beta_{1}^{*}>0$ .
\end{lemma}
\begin{proof}
Define the following function:
\[
\beta_{2}^{+}(t)=\int_{\mb x}\frac{1}{C_{g}}\exp \left(-g\left( \frac{1}{\sigma} \|\left(x_{1}-t\beta_{1}^{*},x_{2}-\beta_{2}^{*},x_{3},\ldots,x_{d}\right)\|_2 \right) \right)x_{2}\tanh(0.5 F_{\mb\beta,\sigma}(\mb x))\dup {\mb x}.
\]
$\beta_{2}^{+}=\beta_{2}^{+}(1)$. We observe that $\beta_{2}^{+}(0)=0$
since the integrand is an odd function in $x_{1}$. The mean value
theorem tells us:
\[
\beta_{2}^{+}=\frac{\partial}{\partial t}\beta_{2}^{+}(t)\mid_{t\in(0,1)}.
\]
Under the regularity condition, we can differentiate inside the integral and obtain the following expression for the derivative of $\beta_2^+(t)$ with respect to $t$:
\begin{align*}
& \frac{\partial}{\partial t}\beta_{2}^{+}(t)\\
= & \beta_{1}^{*}\cdot\int_{x_{\bar{2}}}\int_{x_{2}\geq0}\frac{1}{\sigma^d C_{g}}\left(f\left(\frac{1}{\sigma}\left(x_{1}-t\beta_{1}^{*},x_{2}-\beta_{2}^{*},x_{3},\ldots,x_{d}\right)\right)
-f\left(\frac{1}{\sigma} \left(x_{1}-t\beta_{1}^{*},x_{2}+\beta_{2}^{*},x_{3},\ldots,x_{d}\right) \right) \right) \cdot \\
 &0.5x_{2}\left(\frac{\partial}{\partial x_{1}}F_{\mb\beta,\sigma}(\mb x)\right)\tanh^{\prime}(0.5F_{\mb\beta,\sigma}(\mb x)) \dup x_{2}\dup x_{-2}.
\end{align*}
Since $g$ is convex in $x_{1}$, the gradient with respect to $x_1$ is non-decreasing. It follows that $\frac{\partial }{\partial x_{1}}F_{\mb\beta,\sigma}(\mb x)\geq 0$
as $x_{1}+\|\mb\beta\|_{2}\geq x_{1}-\|\mb\beta\|_{2}$. 
Moreover, similar to the proof of Lemma \ref{lem:derivative-opt-param-1}, we can show that $\frac{\partial}{\partial  x_{1}}F_{\mb\beta}(\mb x)>0$ on a set with positive measure when $\|\mb\beta\|_2>0$. The difference term of the density function:
\begin{align*}
    f\left(\frac{1}{\sigma}\left(x_{1}-t\beta_{1}^{*},x_{2}-\beta_{2}^{*},x_{3},\ldots,x_{d}\right)\right)-f\left(\frac{1}{\sigma} \left(x_{1}-t\beta_{1}^{*},x_{2}+\beta_{2}^{*},x_{3},\ldots,x_{d}\right) \right)>0,
\end{align*}
when $\beta_{2}^{*}>0$ and $x_2\geq 0$. We thus conclude that the integrand is positive when $\beta_1^*>0$ and $\beta_2^*>0$, and the conclusion follows immediately.
\end{proof}

\begin{lemma}[Fixed Point Structure in span($\mb \beta^*$)]
\label{lem:optimal-direction-convergence}Suppose that the density function $f$ satisfies the regularity condition, $\mb 0,\mb\beta^{*}$ and $-\mb\beta^{*}$
are the only fixed points of the \variant EM update in span$(\mb\beta^{*}).$
\end{lemma}

\begin{proof}
By Lemma \ref{lem:two-dim-structure}, span($\mb \beta^*$) is an invariant subspace. We only need to consider $\beta_1^+$, which makes the problem one dimensional. 
\begin{align}
\beta_1^+ = \int_{\mb x} \frac{1}{\sigma^d C_g} \exp\Bigg(-g\left(\frac{1}{\sigma}\|(x_1-\|\mb\beta^*\|_2,x_2,\ldots,x_d)\|_2\right)\Bigg) x_1\tanh(0.5 F_{\mb\beta,\sigma}(\mb x)) \dup {\mb x}.
\end{align}
$\mb 0$ is a trivial fixed point and we assume that $\|\mb\beta\|>0$ in the following.
Conditioning on $x_2,\ldots,x_d$, $g\left(\frac{1}{\sigma}\|(x_1,x_2,\ldots,x_d)\|_2\right)$ is an even convex function in $x_1$, and it is strictly increasing when $x_1\geq 0$. Theorem \ref{thm:Global-Convergence-1D-VEM} tells us that:
\begin{align*}
    &\Big|\int_{x_1} \frac{1}{\sigma C_{2:d}} \exp\Bigg(-g\left(\frac{1}{\sigma}\|(x_1-\|\mb\beta^*\|_2,x_2,\ldots,x_d)\|_2\right)\Bigg) x_1\tanh(0.5 F_{\mb\beta,\sigma}(\mb x)) \dup x_1 - \|\mb\beta^*\|_2 \Big| \\
    \leq &\kappa_{2:d}(\|\mb\beta\|_2,\|\mb\beta^*\|_2,\sigma)\big |\|\mb\beta\|_2-\|\mb\beta^*\|_2 \big|,
\end{align*}
for some $\kappa_{2:d}(\|\mb\beta\|_2,\|\mb\beta^*\|_2,\sigma)\in (0,1)$. $C_{2:d}$ is the normalization factor for the density that is proportional to 
\[
\exp\Bigg(-g\left(\frac{1}{\sigma}\|(x_1-\|\mb\beta^*\|_2,x_2,\ldots,x_d)\|_2\right)\Bigg)\]
conditioned on $x_2,\ldots,x_d$. Now integrating over $x_2,\ldots,x_d$, we get $\big|\beta_1^+-\|\mb\beta^*\|_2 \big|<\big|\|\mb\beta\|_2-\|\mb\beta^*\|_2\big |$ for all $\|\mb\beta\|_2>0$. The conclusion follows.  
\end{proof}

\paragraph{Establishing Local Convergence}
In the following, we denote $B(\mb\beta^*,\sigma)$ as the bound for the \variant EM update. By Cauchy-Schwartz and $|\tanh(\cdot)|\leq 1$, we know that
\begin{align*}
    \|\beta_i^+\|_2\leq \sqrt{\mathbb{E}_{\mb X\sim f_{\mb\beta^*,\sigma}}X_i^2}\quad \forall i. 
\end{align*}
Since the \variant EM update is a two dimensional object, we can bound $\mb\beta^+$ by
\begin{align*}
    \sqrt{\mathbb{E}_{\mb X \sim f_{\mb\beta^*,\sigma}}( X_1^2+ X_2^2)}:=B(\mb\beta^*,\sigma).
\end{align*}
\begin{lemma}[Along $\hat{\mb \beta}$]
\label{lem:along convegence}
Suppose that the density $f$ satisfies the regularity condition. We further assume that 
\[
\sup_{t\in [0,1], \|\mb\beta\|\leq B(\mb\beta^*,\sigma)} \int\frac{1}{\sigma^d C_{g}}(\exp \left(-g \left(\frac{1}{\sigma}\|(x_{1}-\beta_{1}^{*},x_{2}-t\beta_{2}^{*},\ldots,x_{d})\|_2\right)\right)\frac{\partial}{\partial x_{2}}\left[x_{1}\tanh\left(0.5F_{\mb\beta,\sigma}(\mb x)\right)\right]\dup {\mb x}
\]
 is bounded by $D_1(\mb\beta^*,\sigma)$ in absolute value. When $\|\mb\beta\|_2>0$ and $\beta_{1}^*>0$, the \variant EM update satisfies: $\beta_1^+>0$ and 
\[
|\beta_{1}^{+}-\beta_1^{*}|\leq\kappa_1(\mb\beta^{*},\mb\beta,\sigma)\big|\|\mb\beta\|_2-\beta_{1}^{*}\big|+D_1(\mb\beta^*,\sigma)\beta_{2}^{*}.
\]
for some $\kappa_1(\mb\beta,\mb\beta^{*},\sigma)\in (0,1)$. 
\end{lemma}
\begin{proof}
Recall that:
\begin{align*}
\beta_{1}^{+} & = \int_{x_{1}\geq0}\frac{1}{\sigma^{d}C_{g}}\Bigg(\exp \left(-g \left (\frac{1}{\sigma}\|(x_{1}-\beta_{1}^{*},x_{2}-\beta_{2}^{*},x_{3},\ldots,x_{d})\|_2\right)\right)\\
& +\exp \left(-g \left(\frac{1}{\sigma}\|(x_{1}+\beta_{1}^{*},x_{2}-\beta_{2}^{*},x_{3},\ldots,x_{d})\|_2\right)\right)\Bigg)  \cdot x_{1}\tanh\left(0.5F_{\mb \beta,\sigma}(\mb x)\right)\dup {\mb x}.
\end{align*}
When $\|\mb\beta\|_2>0$, the integrand is strictly positive as 
\begin{align}
g\left(\frac{1}{\sigma}\|(x_{1}+\|\mb\beta\|,x_2,\ldots,x_d)\|_2\right)-g\left(\frac{1}{\sigma}\|(x_{1}-\|\mb\beta\|_2,x_2,\ldots,x_d)\|_2\right)>0
\end{align}
over the region where $x_{1}>0$. Thus
\[
\beta_{1}^{+}>0.
\]
A natural corollary is that $\|\mb\beta^{+}\|_2>0$ whenever $\|\mb\beta\|_2>0$.

Let us consider the following (slightly modified) quantity:
\begin{align*}
\beta_{1}^{++} & =\int\frac{1}{\sigma^d C_{g}}\exp \left(-g \left(\frac{1}{\sigma}\|(x_{1}-\beta_{1}^{*},x_{2},\ldots,x_{d})\|_2\right) \right) x_{1}\tanh\left(0.5 F_{\mb\beta,\sigma}(\mb x)\right)\dup {\mb x}\\
 & =\int_{x_{-1}}\int_{x_{1}}\frac{1}{\sigma^d C_{g}}\exp \left(-g \left(\frac{1}{\sigma}\|(x_{1}-\beta_{1}^{*},x_{2},\ldots,x_{d})\|_2\right)\right)x_{1}\tanh\left(0.5 F_{\mb\beta,\sigma}(\mb x)\right)\dup x_{1} \dup x_{-1}.
\end{align*}
It is easy to see that conditioning on $x_{2},\ldots,x_{d}$, the
inner integral is a one-dimensional \variant EM operator with current
estimate $\|\mb\beta\|_2$ and the true parameter $\beta_{1}^{*}$.
Applying Theorem \ref{thm:Global-Convergence-1D-VEM}, we have 
\begin{align*}
 & \Big|\int_{x_{1}}\frac{1}{\sigma C_{2:d}}\exp \left(-g\left(\frac{1}{\sigma}\|(x_{1}-\beta_{1}^{*},x_{2},\ldots,x_{d})\|_2\right)\right)x_{1}\tanh\left(0.5F_{\mb\beta,\sigma}(\mb x)\right)\dup x_{1}-\beta_{1}^{*}\Big|\\
\leq & \kappa_{2:d}(\min(\beta_{1}^{*},\|\mb\beta\|_2),\sigma)\big|\|\mb\beta\|_2-\beta_{1}^{*}\big |,
\end{align*}
where $\kappa_{2:d}(\min(\beta_{1}^{*},\|\mb\beta\|_2),\sigma)<1$ is a contraction
factor depending on $x_{2},\ldots,x_{d}$. Integrating
over $x_{2},\ldots,x_{d}$, we obtain
\begin{align}
\label{eq:contract_part}
|\beta_{1}^{++}-\beta_{1}^{*}|\leq\kappa_1(\min(\beta_{1}^{*},\|\mb\beta\|_2),\sigma)\big|\|\mb\beta\|_2-\beta_{1}^{*}\big|.
\end{align}
Next we bound $\beta_{1}^{+}-\beta_{1}^{++}$. The regularity condition
allows us to change the order of differentiation and integral. 
\begin{align}
 & \beta_{1}^{+}-\beta_{1}^{++} \nonumber \\
= & \frac{\partial}{\partial t}\left[\int\frac{1}{\sigma^d C_{g}}(\exp\left(-g\left(\frac{1}{\sigma}\|(x_{1}-\beta_{1}^{*},x_{2}-t\beta_{2}^{*},\ldots,x_{d})\|_2\right)\right)x_{1}\tanh\left(0.5 F_{\mb\beta,\sigma}(\mb x)\right)\dup {\mb x}\right] \mid_{t\in(0,1)} \nonumber \\
= & \beta_{2}^{*}\left[\int\frac{1}{\sigma^d C_{g}}(\exp\left(-g\left(\frac{1}{\sigma}\|(x_{1}-\beta_{1}^{*},x_{2}-t\beta_{2}^{*},\ldots,x_{d})\|_2\right)\right)\frac{\partial}{\partial x_{2}}\left[x_{1}\tanh\left(0.5F_{\mb\beta,\sigma}(\mb x)\right)\right]\dup {\mb x}\right]\nonumber \\
\leq & D_{1}(\mb\beta^*,\sigma)\beta_{2}^{*} \label{eq:along_drift},
\end{align}
where in the last step, we used the assumption that the integral is uniformly bounded by $D_{1}(\mb\beta^*,\sigma)$, which only depends on $\mb\beta^*$ and $\sigma$ (This assumption is usually satisfied by the regularity condition.) 

Combining \eqref{eq:contract_part} and \eqref{eq:along_drift}, we can bound $\beta_1^+- \beta_1^*$ as follows:
\begin{align}
|\beta_1^+-\beta_1^*| = & |\beta_1^{++}-\beta_1^* +\beta_1^+- \beta_1^{++}| \nonumber \\
\leq &  |\beta_1^{++}-\beta_1^*|+|\beta_1^+- \beta_1^{++}| \nonumber \\
\leq & \kappa_1(\min(\beta_{1}^{*},\|\mb\beta\|),\sigma)\big|\|\mb\beta\|_2-\beta_{1}^{*}\big|+D_{1}(\mb\beta^*,\sigma)\beta_{2}^{*},
\end{align}
and the conclusion follows.
\end{proof}
\begin{lemma}[Orthogonal to $\mb \beta^*$]
\label{lem:orthogonal_convergence}
Assume that the density function $f$ satisfies the regularity condition. We further assume that  
\begin{align*}
\sup_{t\in [0,1], \|\mb\beta\|\in B(\mb\beta^*,\sigma)} \int\frac{1}{\sigma^d C_{g}}\exp\left(g\left(\frac{1}{\sigma}\|(x_{1}-\beta_{1}^{*},x_{2}-t\beta_{2}^{*},\ldots,x_{d})\|_2\right)\right) \Big|x_{2}\frac{\partial}{\partial x_{2}}\left[0.5 F_{\mb\beta,\sigma}(\mb x)\right]\Big| \dup {\mb x}
\end{align*}
is uniformly bounded by  $D_{2}(\mb\beta^*,\sigma )$. The following holds:
\[
|\beta_{2}^{*}-\beta_{2}^{+}|\leq\kappa_{2}(\mb\beta,\mb\beta^{*},\sigma)\beta_{2}^{*}+D_{2}(\mb\beta^*,\sigma )\beta_{2}^{*}
\]
for some $\kappa_{2}(\mb\beta,\mb\beta^{*},\sigma)\in(0,1)$.
\end{lemma}
\begin{proof}
Recall that
\begin{align*}
\beta_{2}^{+}= & \int\frac{1}{\sigma^d C_{g}}\exp \left(-g\left(\frac{1}{\sigma}\|(x_{1}-\beta_{1}^{*},x_{2}-\beta_{2}^{*},\ldots,x_{d})\|_1\right)\right) \cdot x_{2}\tanh\left(0.5F_{\mb\beta,\sigma}(\mb x)\right)\dup {\mb x}.
\end{align*}
Consider the following quantity:
\begin{align*}
\beta_{2}^{++}:=  & \int\frac{1}{\sigma^d C_{g}}\exp \left(-g\left(\frac{1}{\sigma}\|(x_{1}-\beta_{1}^{*},x_{2},\ldots,x_{d})\|_2\right)\right)\cdot x_{2}\tanh\left(0.5 F_{\mb\beta,\sigma}(\mb x)\right)\dup {\mb x} \\
=& \int\frac{1}{2\sigma^dC_{g}}\Bigg(\exp\left(-g\left(\frac{1}{\sigma}\|(x_{1}-\beta_{1}^{*},x_{2},\ldots,x_{d})\|_2\right)\right)+\exp\left(-g\left(\frac{1}{\sigma}\|(x_{1}+\beta_{1}^{*},x_{2},\ldots,x_{d})\|_2\right)\right)\Bigg)\cdot\\
 & \cdot x_{2}\tanh\left(0.5 F_{\mb\beta,\sigma}(\mb x)\right)\dup {\mb x}=0.
\end{align*}
In the last step, we first use the fact that the integrand is even function in $\mb x$, and secondly the integrand is an odd function in $x_1$. Using the mean value theorem, we have:
\[
\beta_{2}^{+}=\frac{\partial}{\partial t}\int\frac{1}{\sigma^d C_{g}}\exp \left(-g\left(\frac{1}{\sigma}\|(x_{1}-\beta_{1}^{*},x_{2}-t\beta_{2}^{*},\ldots,x_{d})\|_2\right)\right)x_{2}\tanh\left(0.5 F_{\mb\beta,\sigma}(\mb x)\right) \dup {\mb x} \mid_{t\in[0,1]}.
\]
Under the regularity condition, we can interchange the order of differentiation and the integral:
\begin{align}
\beta_{2}^{+}= & \beta_{2}^{*}\cdot \int\frac{1}{\sigma^d C_{g}}\exp \left(-g \left(\frac{1}{\sigma}\|(x_{1}-\beta_{1}^{*},x_{2}-t\beta_{2}^{*},\ldots,x_{d})\|_2\right)\right)\tanh\left(0.5 F_{\mb\beta,\sigma}(\mb x)\right)\dup {\mb x} \nonumber\\
 & +\beta_{2}^{*}\cdot\int\frac{1}{\sigma^d C_{g}}\exp\left(-g\left(\frac{1}{\sigma}\|(x_{1}-\beta_{1}^{*},x_{2}-t\beta_{2}^{*},\ldots,x_{d})\|_2\right)\right)x_{2}\frac{\partial}{\partial x_{2}}\left[0.5 F_{\mb\beta,\sigma}(\mb x)\right]\dup {\mb x} \label{eq:orthog_decomp}.
\end{align}
Let us define:
\begin{align*}
\kappa_2 (\|\mb\beta\|,\mb\beta^*,\sigma): = \sup_{t\in (0,1)} \int\frac{1}{\sigma^d C_{g}}\exp \left(-g \left(\frac{1}{\sigma}\|(x_{1}-\beta_{1}^{*},x_{2}-t\beta_{2}^{*},\ldots,x_{d})\|_2\right)\right)\tanh\left(0.5 F_{\mb\beta,\sigma}(\mb x)\right)\dup {\mb x}.
\end{align*}
The first term for $\beta_2^+$ \eqref{eq:orthog_decomp} is bounded by $\kappa_{2}(\mb\beta,\mb\beta^{*},\sigma)\beta_{2}^{*}$ and the second term of $\beta_2^+$ \eqref{eq:orthog_decomp} is bounded by $D_{2}(\mb\beta^*,\sigma )\beta_{2}^{*}$ by the assumption. The conclusion follows.
\end{proof}

\begin{proposition}[Local Quantitative bound]
\label{prop:local_quantitative}
Suppose that the density function $f$ satisfies the regularity condition, and further assume that the conditions in Lemma \ref{lem:along convegence} and \ref{lem:orthogonal_convergence} holds. Then there exists $H(\mb\beta^*,\sigma)$ such that the following holds:
\begin{align}
\|\mb\beta^+ - \mb\beta^* \|_2^2 \leq \kappa(\mb\beta,\mb\beta^*,\sigma)^2 \|\mb\beta-\mb\beta^*\|_2^2 +H(\mb\beta^*,\sigma) \sin(\angle(\mb\beta,\mb\beta^*)).
\end{align}
where $\kappa(\mb\beta,\mb\beta^*,\sigma)<1$. 
\end{proposition}
\begin{proof}
Since the \variant EM update is bounded, we can assume that $\|\mb\beta\|_2$ is bounded without loss of generality. Furthermore, $\|\mb\beta\|_2\leq B(\mb\beta^*,\sigma)$.
\begin{align*}
\|\mb\beta^+ - \mb\beta^* \|_2^2= & \|\beta_1^+ - \beta_1^* \|_2 + \|\beta_2^+ - \beta_2^* \|_2\\
\leq & (\kappa_1 |\|\mb\beta\|_2-\beta_1^*|+D_1 \beta_2^*)^2 + (\kappa_2 \beta_2^* + D_2\beta_2^*)^2 \\
= & \kappa_1^2 |\|\mb\beta\|_2-\beta_1^*|^2 + 2D_1\beta_2^*\kappa_1 |\|\mb\beta\|_2-\beta_1^*| + D_1^2(\beta_2^*)^2 +\\
& \kappa_2^2  (\beta_2^*)^2+ 2D_2\beta_2^*\kappa_2 \beta_2^* + D_2^2 (\beta_2^*)^2 \\
\leq &\kappa^2 \|\mb\beta - \mb\beta^* \|^2 + H(\mb\beta^*,\sigma) \sin(\angle(\mb\beta,\mb\beta^*)).
\end{align*}
In the last step, $\kappa = \max(\kappa_1,\kappa_2)$, and $H$ absorbs all the coefficient of $\beta_2^*$ in the cross term. It is easy to check that $H$ only depends on $\mb\beta^*$ and $\sigma$ as $\|\mb\beta\|_2$ is bounded by $B(\mb\beta^*,\sigma)$.
\end{proof}

\paragraph{Establishing Local Stable Region}
Define
\begin{align}
\kappa_\tau := \max_{\mb\beta \in N_{\mb\beta^*}(\tau)} \kappa(\mb\beta,\mb\beta^*,\sigma)
\end{align}
be the worst-case contraction factor in the $\tau$-neighborhood of $\mb\beta^*$, namely $N_{\mb\beta^*}(\tau) = \left\{\mb\beta: \|\mb\beta - \mb\beta^* \|_2 \leq \tau \|\mb\beta^*\|_2\right\}$.

\begin{corollary}[Local Stable Region]
\label{cor:non-escaping}
Let $D$ be a positive number satisfying $D\leq \frac{1}{2}\|\mb\beta^*\|_2$. Suppose that $\mb\beta^0 \in \bb{R}^d$ is such that (1) $\sin(\angle (\mb\beta^0,\mb\beta^*)) \leq \frac{D^2(1-\kappa_{0.5}^2)}{H(\mb\beta^*,\sigma)}$ (2) $\|\mb\beta^0-\mb\beta^*\|_2\leq D$, where $H(\mb\beta^*,\sigma)$ is defined in Corollary~\ref{prop:local_quantitative}. The following holds for the \variant EM update $\mb\beta^1$ starting at $\mb\beta^0$:
\begin{align*}
\|\mb\beta^1-\mb\beta^*\|_2\leq D,\quad \sin(\angle (\mb\beta^1,\mb\beta^*)) \leq \frac{D^2(1-\kappa_{0.5}^2)}{H(\mb\beta^*,\sigma)}
\end{align*}
\end{corollary}
\begin{proof}
Applying Proposition \ref{prop:local_quantitative}, we have
\begin{align*}
\|\mb\beta^1 - \mb\beta^* \|^2 \leq  & \kappa(\mb\beta^0,\mb\beta^*,\sigma)^2 \|\mb\beta-\mb\beta^*\|_2^2 +H(\mb\beta^*,\sigma) \sin(\angle(\mb\beta,\mb\beta^*)) \\
\leq & \kappa_{0.5}^2 D^2 +H(\mb\beta^*,\sigma)  \frac{D^2(1-\kappa_{0.5}^2)}{H(\mb\beta^*,\sigma)} = D^2.
\end{align*}
On the other hand, by the angle decreasing property, $\sin(\angle (\mb \beta^1,\mb\beta^*))\leq \sin(\angle (\mb \beta,\mb\beta^*))\leq \frac{D^2(1-\kappa_{0.5}^2)}{H(\mb\beta^*,\sigma)}$, thus the corollary is proved.
\end{proof}

\paragraph{Establishing Norm Incerasing Region}

\begin{lemma}
\label{lem:norm_increasing}
There exists $\theta>0$ such that $\|\mb \beta^+\|_2>\|\mb \beta\|_2$ whenever $\mb \beta\in N_0(\frac{1}{8}\|\mb \beta^*\|_2)$ and $\angle (\mb \beta^*, \mb \beta) \leq \theta$. 
\end{lemma}

\begin{proof}
Let us recall the expression for $\beta_1^+$ and $\beta_1^{++}$ (defined in Lemma \ref{lem:along convegence}):
\begin{align*}
\beta_1^+ = & \int_{x_{-1}}\int_{x_{1}}\frac{1}{\sigma^d C_{g}}\exp \left(-g \left(\frac{1}{\sigma}\|(x_{1}-\|\mb \beta^{*}\|_2\cos(\angle (\mb \beta^*, \mb \beta)),x_{2}-\|\mb \beta^{*}\|_2\sin(\angle (\mb \beta^*, \mb \beta)),\ldots,x_{d})\|_2\right)\right)\\
& \cdot x_{1}\tanh\left(0.5 F_{\mb\beta,\sigma}(\mb x)\right)\dup x_{1} \dup x_{-1},
\end{align*}
and
\begin{align*}
\beta_1^{++} = & \int_{x_{-1}}\int_{x_{1}}\frac{1}{\sigma^d C_{g}}\exp \left(-g \left(\frac{1}{\sigma}\|(x_{1}-\|\mb \beta^{*}\|_2\cos(\angle (\mb \beta^*, \mb \beta)),x_{2},\ldots,x_{d})\|_2\right)\right)\\
& \cdot x_{1}\tanh\left(0.5 F_{\mb\beta,\sigma}(\mb x)\right)\dup x_{1} \dup x_{-1}.
\end{align*}

We view $\beta_1^{++}$ as a modified iterate compared to $\beta_1^{+}$, and their discrepancy becomes smaller and smaller as $\angle (\mb \beta^*, \mb \beta)$ goes to $0$ ($\beta_1^{++}=\beta_1^+$ when $\angle (\mb \beta^*, \mb \beta)=0$). For the modified iterate $\beta_1^{++}$, it has a fixed point $\|\mb \beta^{*}\|_2\cos(\angle (\mb \beta^*, \mb \beta))$ along the $\hat{\mb \beta}$ direction. Moreover, $\beta_1^{++}$ has the following two properties inherited from the structure of an one dimensional update:
\begin{itemize}
    \item $\beta_1^{++}$ is increasing in $\|\mb \beta\|_2$;
    \item $\beta_1^{++}>\|\mb \beta\|_2$ whenever $\|\mb \beta\|_2<\|\mb \beta^{*}\|_2\cos(\angle (\mb \beta^*, \mb \beta)) $.
    \begin{figure}[h]
        \centering
        \includegraphics[scale = 0.4, trim = 200 200 200 200, clip]{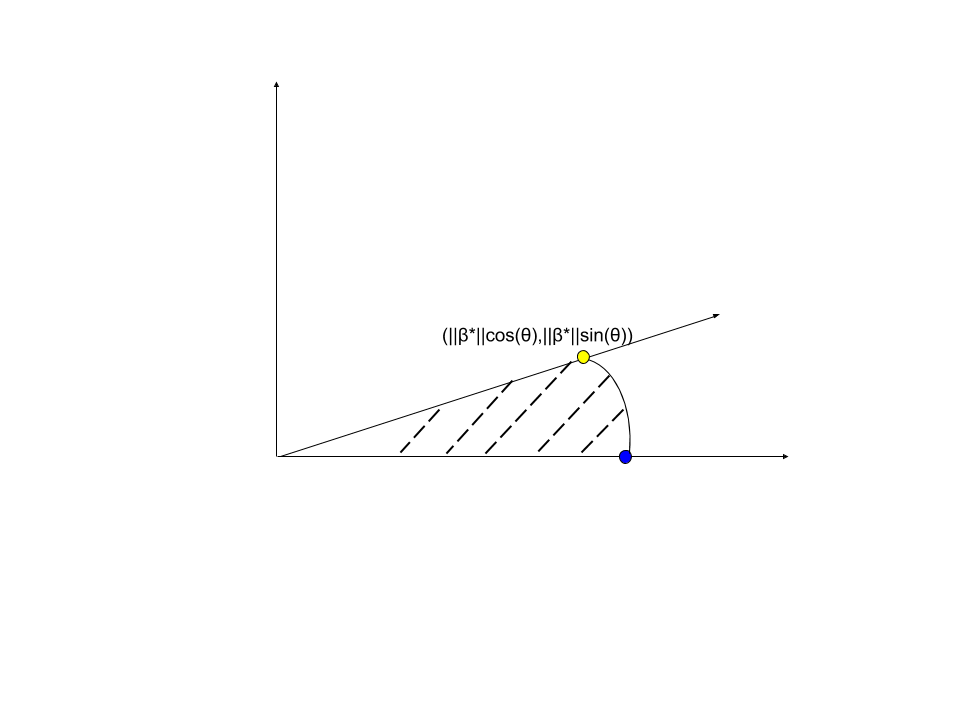}
        \caption{The shaded region $S(\theta)$ is where $\beta_1^{++}>\|\mb \beta\|_2$.}
        \label{fig:norm_increasing_region}
    \end{figure}
\end{itemize} 
In Figure \ref{fig:norm_increasing_region}, we illustrate the norm increasing region for $\beta_1^{++}$:
\[
S(\theta): = \left\{\mb \beta: \beta_1^{++}>\|\mb \beta\|_2, \angle (\mb \beta^*, \mb \beta) \leq \theta\right\}.
\]

By the continuity of the \variant EM update, it is easy to see that the norm increasing region for $\beta_1^+$:
\[
\tilde{S}(\theta): = \left\{\mb \beta: \beta_1^{+}>\|\mb \beta\|_2, \angle (\mb \beta^*, \mb \beta) \leq \theta\right\}
\]
is close to $S(\theta)$ when $\theta$ is small (note that $S(\theta) = \tilde{S}(\theta)$ when $\theta = 0$). Since for some $\Theta_0>0$, $N_0(1/8\|\mb \beta^*\|_2)\cap \left\{\mb \beta: \angle (\mb \beta^*, \mb \beta) \leq \theta \right\}\subseteq S(\Theta_0)$, i.e, $(\Theta_0)$ contains a (bounded) cone-shape region. We conclude that for sufficiently small $\theta$, $N_0(1/8\|\mb \beta^*\|_2)\cap \left\{\mb \beta: \angle (\mb \beta^*, \mb \beta) \leq \theta \right\}\subseteq \tilde{S}(\theta)$. In plain words, $\tilde{S}(\theta)$ contains all those $\mb \beta$, whose angle with $\mb \beta^*$ is less than $\theta$ and whose norm is less than $\frac{1}{8}\|\mb \beta^*\|_2$.
\end{proof}

\begin{proposition}
Suppose ${\mb \beta^t}\subseteq N_0(\frac{1}{8}\|\mb \beta^*\|_2)$ is a sequence of \variant EM iterates and $\mb \beta^t \neq 0$ for all $t$, then it is impossible that $\lim_t \mb \beta^t = \mb 0$.   
\end{proposition}
\begin{proof}
We argue by contradiction. By the angle decreasing property of the iterates, there exists $\theta>0$ and $T$, such that $\angle (\mb \beta^*, \mb \beta^t) \leq \theta$ for all $t\geq T$. By Lemma \ref{lem:norm_increasing}, we know that $\|\mb \beta^{t+1}\|_2> \|\mb \beta^{t}\|_2\geq \mb \|\mb \beta^{T}\|_2>0$ for all $t\geq T$. Thus, the norm of the iterates is lower bounded by a positive number and it is impossible for the iterates to converge to $\mb 0$. 
\end{proof}

\subsection{Proof of Lemma~\ref{lem: non-contraction}}
\label{subsec:non-contraction}

For readability we restate the lemma below.

\noncontractelltwo*

\begin{proof}
$\beta_1^+$ is an increasing function in $\|\mb \beta\|_2$ by Lemma \ref{lem:increasing_wrt_beta}.
Let us understand the derivative of $\frac{\partial\beta_{1}^{+}}{\partial\|\mb\beta\|_{2}}\mid_{\|\mb\beta\|_{2}=0}:=D_{1}(\|\mb\beta^{*}\|_{2},f)$ when $\angle\mb\beta,\mb\beta^{*}=\frac{\pi}{2}$. The expression is the following:
\begin{align}
 D_{1}(\|\mb\beta^{*}\|_{2},f) & =\int_{\mb x}\frac{1}{C_{g}}\exp \left(-g \left(\frac{1}{\sigma}\|(x_{1},x_{2}-\|\mb\beta^{*}\|,\ldots,x_{d})\|_2\right)\right)\frac{x_{1}}{\sigma}\frac{\partial }{\partial x_1}g \left(\frac{1}{\sigma}\|(x_{1},\ldots,x_{d})\|_2\right) \dup {\mb x}\nonumber \\
 & =\int_{\mb x}\frac{1}{C_{g}}\exp (-g(\|(x_{1},x_{2}-\|\mb\beta^{*}\|,\ldots,x_{d})\|_2)x_{1} \frac{\partial }{\partial x_1}g(\|(x_{1},\ldots,x_{d})\|_2)\dup {\mb x}.\label{eq:derivative-at-boundary}
\end{align}
Note that when $\|\mb\beta^{*}\|=0$, $D_{1}=1$. Let us further take the derivative with respect to $\|\mb\beta^{*}\|_{2}$:
\begin{align}
\frac{d}{d\|\mb\beta^{*}\|_{2}}D_{1}(\|\mb\beta^{*}\|_{2},f)= & \mathbb{E}_{\mb X\sim f_{\mb\beta^{*}}}X_{1} \frac{\partial ^2}{\partial X_1 \partial X_2}g(\|(X_{1},\ldots,X_{d})\|_2).
\end{align}

Here $f_{\mb\beta^{*}}(\mb x)=\frac{1}{C_g}\exp(-g(\|(x_1,x_2-\beta_2^*,x_3,\ldots,x_d)\|_2))$ (slightly different from the previous sections). In the special case where $g(\mb x) = c\|\mb x\|_2^r$ for some $r\geq 1$ and $c>0$, 
\begin{align}
\frac{d}{d\|\mb\beta^{*}\|_{2}}D_{1}(\|\mb\beta^{*}\|_{2},f) = \mathbb{E}_{\mb X\sim f_{\mb\beta^{*}}} r(r-2)X_1^2X_2\|\mb X\|_2^{\frac{r}{2}-2}.
\end{align}
Since the integrand is an odd function in $X_2$, Lemma \ref{lem:positive_integral} tells us that the above derivative is positive when $r\geq 2$. When $r<2$, the above derivative is negative. We deduce that when $r<2$, $ D_{1}(\|\mb\beta^{*}\|_{2},f)<1$, and when $r>2$, $ D_{1}(\|\mb\beta^{*}\|_{2},f)>1$. In particular, when $r>2$, there is a positive fixed point $\tilde{\mb\beta}$ (i.e., $\tilde{\mb \beta}_1^{+} = \|\tilde{\mb \beta}\|$) for the \variant EM operator in the orthogonal axis. When $\mb\beta$ is in the orthogonal direction to $\mb\beta^*$ with $\|\mb\beta\|_2<\|\tilde{\mb\beta}\|_2$, $\|\mb\beta^+-\tilde{\mb\beta}\|_2 < \|\mb\beta-\tilde{\mb\beta}\|_2$. In particular, this means that $\|\mb\beta^+-\mb\beta^*\|_2>\|\mb\beta-\mb\beta^*\|_2$.
We have completed the proof of Lemma~\ref{lem: non-contraction}
\end{proof}

\begin{lemma} 
\label{lem:increasing_wrt_beta}
Under the regularity condition, $\beta_1^{+}$ is a strictly increasing function of $\|\mb\beta\|_2$. 
\end{lemma}
\begin{proof}
Note that $\beta_{1}^{+}$ is a function of $\|\mb\beta\|_{2}$ and $\|\mb\beta^{*}\|_{2}$.
We are interested in how $\beta_{1}^{+}$ will change with respect
to $\|\mb\beta\|_{2}$. Under the regularity condition, we can take the derivative with respect to $\|\mb\beta\|_{2}$, which gives
\begin{align}
& \frac{\partial\beta_{1}^{+}}{\partial\|\mb\beta\|_{2}} \nonumber \\ 
=& \int_x \Bigg(0.5\frac{1}{\sigma^dC_g}f\left(\frac{1}{\sigma}(x_{1}-\beta_{1}^{*},x_{2}-\beta_{2}^{*},x_{3},\ldots,x_{d})\right)+0.5\frac{1}{\sigma^dC_g}f\left(\frac{1}{\sigma}(x_{1}+\beta_{1}^{*},x_{2}-\beta_{2}^{*},x_{3},\ldots,x_{d})\right)\Bigg)\cdot \nonumber \\
 & \frac{x_{1}}{2\sigma}\left(g^{\prime}\left(\frac{1}{\sigma}(x_{1}-\|\mb\beta\|,x_{2},\ldots,x_{d})\right)+g^{\prime} \left(\frac{1}{\sigma} (x_{1}+\|\mb\beta\|,x_{2},\ldots,x_{d}) \right)\right )\cdot\tanh^{\prime}(0.5F_{\mb\beta,\sigma}(\mb x)) \dup x \label{eq:derivative_beta}.
\end{align}
We note that the integrand is an even function in $x_{1}$, and it is strictly positive when $x_{1}>0$. Therefore, we have  $\frac{\partial\beta_{1}^{+}}{\partial\|\mb\beta\|_{2}}>0$, hence
$\beta_{1}^{+}$ is an increasing function in $\|\mb\beta\|_2$. 
\end{proof}

\section{Regularity Condition}
\label{appendix-subsec:regularity}
The regularity condition is a technical condition that makes changing the order of differentiation and integration valid. Formally, let us first recall the measure theory statement of Leibniz's integral rule 
\begin{proposition}[Theorem 16.8 of \cite{billingsley2008probability}]
Let $S$ be an open subset of  $\mathbb{R}$, and $\Omega$ be a measure space. Suppose $f:S\times \Omega \to \mathbb{R}$ satisfies the following conditions:
(1) $f(s,\omega)$ is a Lebesgue-integrable function of $\omega$ for each $s\in S$; (2) For almost all $\omega \in \Omega$, the derivative $f_s$ exists for all $s\in S$; (3) There is an integrable function $\theta:\Omega\to \mathbb{R}$ such that $|f_s(s,\omega)|\leq \theta(\omega)$ for all $s\in S$ and almost every $\omega \in \Omega$. It follows that:
\begin{align}
\frac{\dup{} }{\dup {s}} \int_{\Omega} f(s,\omega) \dup \omega = \int_{\Omega} f_s(s,\omega) \dup \omega.
\end{align} 
\end{proposition} 
In the above proposition, $S$ is the parameter space, and $\omega$ is the random variable. 
Recall the \variant EM update function:
\begin{align*}
M(\mb \beta^*,\mb \beta)= & \bb{E}_{\mb X\sim f_{\mb \beta^*,\sigma}} \mb X \tanh(0.5 F_{\mb \beta,\sigma}(\mb X)),\\
F_{\mb \beta,\sigma}(\mb X) = & g\left(\frac{1}{\sigma}\|X+\mb \beta\|_2 \right)-g\left(\frac{1}{\sigma}\|X-\mb \beta\|_2 \right).
\end{align*}
Using the rotation invariant property of the distribution, we adopt a local orthogonal basis as in Lemma~\ref{lem:two-dim-structure}. The above two functions are equivalent to the following:
\begin{align*}
M(\mb \beta^*,\mb \beta) = & \int_{\mb x}\frac{1}{\sigma^d C_g}\exp\left(-g\left(\|(x_1-\beta_1^*,x_2 - \beta_2^*,x_3,\ldots,x_d )\|_2\right)\right)\mb x \tanh(0.5 F_{\mb \beta,\sigma}(\mb x)) \dup {\mb x}, \\
F_{\mb \beta,\sigma}(\mb x) = & g\left(\frac{1}{\sigma}\|(x_1+\|\mb \beta\|_2,x_2,\ldots,x_d)\|_2 \right)-g\left(\frac{1}{\sigma}\|(x_1-\|\mb \beta\|_2,x_2,\ldots,x_d)\|_2 \right).
\end{align*}
Indeed, from the above representation for $M(\mb \beta^*,\mb \beta)$, we can think of the $M$ function as a function in three variables: $\beta_1^*, \beta_2^*$ and $\|\mb \beta\|_2$. Define, for each $i=1,\ldots,d$,
\begin{align}
&M_i(z_1,z_2,z_2) \nonumber \\
: =  & \int_{\mb x}\frac{1}{\sigma^d C_g}\exp\left(-g\left(\frac{1}{\sigma}\|(x_1-z_1,x_2 - z_2,x_3,\ldots,x_d )\|_2\right)\right)\cdot \nonumber\\
& x_i \tanh\left(0.5 \left(g\left(\frac{1}{\sigma}\|(x_1+z_3,x_2,\ldots,x_d)\|_2 \right)-g\left(\frac{1}{\sigma}\|(x_1-z_3,x_2,\ldots,x_d)\|_2 \right)\right)\right) \dup {\mb x}.   
\end{align}
The regularity condition for $f\in \mathcal{F}$ ensures that for each $i=1,\ldots,d$, $j=1,2,3$, the following holds: 
\begin{align*}
    &\frac{\partial M(z_1,z_2,z_3)}{\partial z_j}\\
    = & \int_{\mb x} \frac{1}{\sigma^d C_g} \frac{\partial }{\partial z_j}\Big[\exp\left(-g\left(\|(x_1-z_1,x_2 - z_2,x_3,\ldots,x_d )\|_2\right)\right)\cdot \\
    & x_i \tanh\left(0.5 \left(g\left(\frac{1}{\sigma}\|(x_1+z_3,x_2,\ldots,x_d)\|_2 \right)-g\left(\frac{1}{\sigma}\|(x_1-z_3,x_2,\ldots,x_d)\|_2 \right)\right)\right) \Big]\dup {\mb x}. 
\end{align*}
In other words, we can differentiate the \variant EM update with the parameter by putting the differentiation operator inside the integral. Note that the main method for analyzing the \variant EM update is the sensitivity analysis, in which we regularly differentiate $M(z_1,z_2,z_3)$ with one of the parameters. 

In view of Leibniz's rule, it suffices to ensure that 3 conditions are satisfied:
\begin{itemize}
    \item For each $z_1,z_2$ and $z_3$, the integrand of $M_i(\cdot)$ is dominated by $|x_i|$ since the $\tanh(\cdot)$ function is uniformly bounded by $1$. $\bb{E}_{\mb X\sim f_{\mb \beta^*,\sigma}}|x_i|<\infty$ for all $i$ when $f$ is a log-concave density.
    \item $g$ is a convex function on $\mathbb{R}^+$, therefore, it is differentiable on $\mathbb{R}^+$ except on a measure $0$ set. $\|\cdot\|_2$ is differentiable except at the origin. Thus, by the composition rule, we infer that the integrand is differentiable with respect to $z_j$ ($j=1,2,3$) for almost all $\mb x\in \bb{R}^d$. 
    \item When we differentiate $M(z_1,z_2,z_3)$ with $z_j$, the parameter space is bounded. We differentiate $M_1$ with respect to $z_1$ in Lemma \ref{lem:along convegence}, and its value is taken between $\min(\|\mb \beta\|_2,\beta_1^*)$ and $\max(\|\mb \beta\|_2,\beta_1^*)$; We differentiate $M_2$ with respect to $z_1$ in Lemma \ref{lem:angle-decrease}, and its value is taken between $0$ and $\beta_1^*$; We differentiate $M_1$ with respect to $z_2$ in Lemma~\ref{lem:along convegence}, and its value is taken between $0$ and $\beta_2^*$; We differentiate $M_2$ with respect to $z_2$ in Lemma \ref{lem:orthogonal_convergence} and its value is between $0$ and $\beta_2^*$; We differentiate $M_1$ with respect to $z_3$ and its value is between $0$ and $\|\mb \beta\|_2$. Since $\|\mb \beta\|_2$ is bounded by a function of $\beta^*$ and $\sigma$, the above parameter space $S_j(\mb \beta^*,\sigma)$ for $z_j$ ($j=1,2,3$) is all bounded. Therefore,it suffices to ensure the integrability of the derivative (with respect to $z_j$). 
\end{itemize}
In order to verify the regularity condition for a log concave distribution, one needs to ensure the derivative of integrand in $M_i$ ($i=1,2$) with respect to $j$ ($j=1,2,3$) is integrable over the parameter space $S_j(\mb \beta^*,\sigma)$. It suffices to show the following quantity is finite: 
\begin{align}
    &\int_{\mb x} \frac{1}{\sigma^d C_g} \sup_{z_j \in S_j(\mb \beta^*,\sigma)}\Bigg|\frac{\partial }{\partial z_j}\Big[\exp\left(-g\left(\|(x_1-z_1,x_2 - z_2,x_3,\ldots,x_d )\|_2\right)\right)\cdot \nonumber \\
    & x_i \tanh\left(0.5 \left(g\left(\frac{1}{\sigma}\|(x_1+z_3,x_2,\ldots,x_d)\|_2 \right)-g\left(\frac{1}{\sigma}\|(x_1-z_3,x_2,\ldots,x_d)\|_2 \right)\right)\right) 
    \Big]\Bigg | \dup {\mb x}
\end{align}  

As an example, consider a general polynomial family, it is easy to verify that the above condition holds as a log concave distribution has finite moments of all the order. 

\section{Finite Sample Analysis}
\label{appendix-finitesmaple}
\label{appendix:finitesample}

In this section, we provide the proofs for analysis of the \vem algorithm in the finite sample case. Proof of Proposition \ref{prop:1d-finite-analysis} is presented in Section \ref{subsec:prop-finite-sample-step}, which establishes an one-iteration bound. Propositions~\ref{prop:first_stage} and~\ref{prop:second_stage} are proved in Section \ref{subsec:global_finite} for the global convergence. In Section \ref{subsec:finite-special-case}, we discuss the implication for some special distributions including Gaussian, Laplace and Logistic. 
\subsection{Proof of Proposition \ref{prop:1d-finite-analysis}}
\label{subsec:prop-finite-sample-step}
\onedfinite*
\begin{proof}
In the 1-d finite sample case, the \variant EM update is 
\begin{align*}
\tilde{M}(\beta^{*},\beta)= 
\frac{1}{n}\sum_{i=1}^{n}x^{i}\tanh(0.5F_{\beta,\sigma}(x^i)) 
\end{align*}
 
Since $|\tanh(\cdot)|\leq 1$, each summand $z^{i}:=x^{i}\tanh(0.5F_{\beta,\sigma}(x^i))$
is a sub-exponential random variable with $\Psi_1$ Orlicz norm upper bounded by $\beta^*+\sigma C_f$. This is because each $x_i\sim f_{\beta^*,\sigma}$ is a sub-exponential random variable (see Lemma \ref{lem:logconcave-subexponential}) with $\Psi_1$ Orlicz norm $O(\beta^*+\sigma C_f)$, where $C_f$ is the $\Psi_1$ Orlicz norm of a random variable with log-concave density $f$. Using Bernstein's inequality from Theorem \ref{appendix-thm:bernstein}, we have
\[
\mathbb{P}\{|\tilde{M}(\beta^{*},\beta) - M(\beta^{*},\beta)|\geq t\}\leq2\left[-cn\min \left(-\frac{t^{2}}{(\beta^*+\sigma C_f)^{2}},\frac{t}{\beta^*+\sigma C_f}\right)\right].
\]
Coupling with the one-step analysis for the population \variant EM update in Theorem \ref{thm:Global-Convergence-1D-VEM}, we can bound the finite sample \variant EM update as follows:
\[
|\tilde{M}(\beta,\beta^{*})-\beta^{*}|\leq\kappa(\beta^*,\beta,\sigma)|\beta-\beta^{*}|+O\left(\sqrt{\frac{(\beta^*+\sigma C_f)^{2}}{n}\log \frac{1}{\delta}}\right)
\]
with probability at least $1-\delta.$
\end{proof}
Let us recall the following equivalent definition for sub-exponential random variables:
\begin{lemma}[Proposition 2.7.1 of \cite{vershynin2018high}]
\label{lem:subexponential}
Let $X$ be a random variable in $\R$. $X$ is sub-exponential iff $\mathbb{E}\exp(|X|/K_{3})\leq2$ for some
$K_{3}>0$. 
\end{lemma}
We first show a random variable with a symmetric log-concave density is necessarily sub-exponential. 
\begin{lemma}
\label{lem:logconcave-subexponential}
If a random variable $X$ has a log-concave density $f$ that is also an even function, then $X$ is sub-exponential, with the $\Psi_1$ Orlicz norm (sub-exponential
norm) depending on log $f$.
\end{lemma}

\begin{proof}
For a general symmetric log-concave distribution with density $f(x)=\frac{1}{C_g}\exp(-g(|x|))$, the set of sub-differential of $g$: $\left\{\partial g(x): x\in \R\right\}$ is non decreasing with respect to $x$. Moreover, the sub-differentials are non-negative when $x\geq 0$. Suppose that $\left\{\partial g(x):x\in \R\right\}$ has an upper bound $C$, we can pick $M$ with $0<M<C$ and $M\in\partial g(x_{0})$
for some $x_{0}>0$. Otherwise, the $\left\{\partial g(x): x\in \R\right\}$ does not have an upper bound and we can pick an arbitrary $M>0$ such that $M\in\partial g(x_{0})$ for some $x_{0}>0$. By definition of the sub-differential, we have
\[
g(x)\geq g(x_{0})+M(x-x_{0}).
\]

Now let us compute the moment generating function: $\mathbb{E}_{X\sim f}\exp(\frac{1}{K}|X|)$.
\begin{align*}
\mathbb{E}_{X\sim f} \exp\left(\frac{1}{K}|X|\right) & =2\int_{x\geq0}\frac{1}{C_g}\exp\left(\frac{1}{K}x\right )\exp(-g(x)) \dup x \\
 & \leq2\int_{x\geq 0}\frac{1}{C_g} \exp\left(\frac{1}{K}x\right)\exp(-g(x_{0})-M(x-x_{0}))\dup x\\
 & =2\int_{x\geq0} \frac{1}{C_g}\exp(-g(x_{0})+Mx_{0})\exp(-(M-\frac{1}{K})x) \dup x\\
 & =2 \frac{1}{C_g}\exp(-g(x_{0})+Mx_{0})\frac{1}{M-\frac{1}{K}}<\infty.
\end{align*}

Using the dominated convergence theorem, we know that
\begin{align*}
    \lim_{K\to \infty} \mathbb{E}_{X\sim f} \exp \left(\frac{|X|}{K}\right) = 0,
\end{align*}
thus, there exists some $K_0$ such that $\mathbb{E}_{X\sim f} \exp(\frac{|X|}{K})\leq 2$. In particular, $X$ is sub-exponential (by Lemma \ref{lem:subexponential}) with a finite $\|\cdot\|_{\Psi_1}$ Orlicz norm.
\end{proof}

Having established the sub-exponential property of the log-concave distribution, we use $C_f$ to denote the $\Psi_1$ Orlicz norm for a log concave distribution $f$. With translation and scaling, it is not hard to see the $\Psi_1$ Orlicz norm for $f_{\beta^*,\sigma}$ is of the order $O(\beta^*)+\sigma C_f(0,1)$.

\begin{theorem}[Bernstein's inequality Theorem 2.8.1 of \cite{vershynin2018high}] 
\label{appendix-thm:bernstein}
Let $X_{1},\ldots,X_{N}$
be independent sub-exponential random variables. Then for every $t\geq0$,
we have
\[
\mathbb{P}\{|\sum_{i=1}^{N}X_{i}|\geq t\}\leq2\exp\left[-c\min\left(\frac{t^{2}}{\sum_{i=1}^{N}\|X_{i}\|_{\psi_{1}}^{2}},\frac{t}{\max_{i}\|X_{i}\|_{\psi_{i}}}\right)\right],
\]
where $c>0$ is an absolute constant. 
\end{theorem}

\subsection{Proofs of Proposition~\ref{prop:first_stage} and Proposition~\ref{prop:second_stage}}
\label{subsec:global_finite}

 For readability, we restate the propositions below.
\firststage*
\secondstage*
\begin{proof}
The premise in Proposition~\ref{prop:first_stage} ensures that conditions in Lemma \ref{lem:l1} and Corollary \ref{cor:greater_than_ell} hold, which guarantee that all the future iterates remain in $(\beta^0,\infty)$. 
There are two stages of analysis for the \vem algorithm in the finite sample case:
\begin{enumerate}
    \item The initial $\beta^0$ is $\in (0,0.5\beta^*)$ or the initial $\beta^0$ is $\in (1.5\beta^*,\infty)$. In this case, the iterates will get into the local stable region $(0.5\beta^*,1.5\beta^*)$ quickly.   
    \item The iterates enters the stable region $(0.5\beta^*,1.5\beta^*)$, and converge to an $\epsilon$-close estimate. \end{enumerate}

Let $\tilde{\beta}^t$ denote the $t$-th iterate. The per iteration bound established in Proposition \ref{prop:1d-finite-analysis} says that with probability at least $1-\delta$:
\begin{align}
\label{eq:one_iter_bound}
|\tilde{\beta}^t-\beta^*|\leq &  \kappa(\beta^*,\tilde{\beta}^{t-1},\sigma)|\tilde{\beta}^{t-1}-\beta^*| + O\left(\sqrt{\frac{(\beta^*+C_f\sigma)^{2}}{n}\log\frac{1}{\delta}}\right).
\end{align}

 Let us first analyze the first stage:

In the case where $\beta^0 \in (0,0.5\beta^*)$, the iterates contracts to $\beta^*$ initially by Lemma \ref{lem:l1}. We use induction from step \eqref{eq:one_iter_bound} to obtain:
\begin{align*}
|\tilde{\beta}^t-\beta^*| \leq & \kappa(\beta^*,\beta^0,\sigma)|\tilde{\beta}^{t-1}-\beta^*| + O\left(\sqrt{\frac{(\beta^*+C_f\sigma)^{2}}{n}}\log \frac{1}{\delta}\right) \\
\leq & \kappa(\beta^*,\beta^0,\sigma)^t|\beta^0-\beta^*| + \frac{1}{1-\kappa(\beta^*,\beta^0,\sigma)}O\left(\sqrt{\frac{(\beta^*+C_f\sigma)^{2}}{n}\log \frac{1}{\delta}}\right). 
\end{align*}
Under the assumption that the size of fresh samples per iteration satisfies $n = \tilde{O}\left(\frac{(1+C_f\eta)^2}{(1-\kappa(\beta^*,\beta^0,\sigma))^2}\right)$, it is guaranteed that the accumulative statistical error is upper bounded: $\frac{1}{1-\kappa(\beta^*,\beta^0,\sigma)}O\left(\sqrt{\frac{C_f(\beta^*,\sigma)^{2}}{n}\log \frac{1}{\delta}}\right)\leq 0.25 \beta^*$. Therefore, after $T = O\left(\frac{\log \frac{0.25\beta^*}{|\beta^0-\beta^*|}}{\log \kappa (\beta^*,\beta^0,\sigma)}\right)$ iterations, $|\tilde{\beta}^t-\beta^*|<0.5 \beta^*$. The probability is at least $1-\delta \cdot O\left(\frac{\log \frac{0.25\beta^*}{|\beta^0-\beta^*|}}{\log \kappa (\beta^*,\beta^0,\sigma)}\right)$ by a union bound.

In the case where $\beta^0 > 1.5\beta^*$, the sample complexity per iteration ensures that all future iterates are lower bounded by $0.5\beta^*$ (see the proof of corollary \ref{cor:greater_than_ell}.) We deduce the following:
\begin{align}
|\tilde{\beta}^t-\beta^*| \leq & \kappa(\beta^*,0.5\beta^*,\sigma)|\tilde{\beta}^{t-1}-\beta^*| + O\left(\sqrt{\frac{(\beta^*+C_f\sigma)^{2}}{n}}\log \frac{1}{\delta}\right) \nonumber \\
\leq &\kappa(\beta^*,0.5\beta^*,\sigma)^t|\beta^0-\beta^*| + \frac{1}{1-\kappa(\beta^*,0.5\beta^*,\sigma)}O\left(\sqrt{\frac{(\beta^*+C_f\sigma)^{2}}{n}\log \frac{1}{\delta}}\right)
\end{align}
Again, the accumulative statistical error is bounded by $0.25\beta^*$. After $T = O\left(\frac{\log \frac{0.25\beta^*}{|\beta^0-\beta^*|}}{\log \kappa(\beta^*,0.5\beta^*,\sigma)}\right)$ iterations, $|\tilde{\beta}^t-\beta^*|<0.5\beta^*$. The probability is at least $1-\delta \cdot O\left(\frac{\log \frac{0.25\beta^*}{|\beta^0-\beta^*|}}{\log \kappa(\beta^*,0.5\beta^*,\sigma)}\right)$ by a union bound. 

Now let us analyze the second stage with the goal of achieving a relative error of $\epsilon$. Since the initial distance to $\beta^*$ is upper bounded by $0.5\beta^*$, it suffices to ensure the following:
\begin{align*}
0.5\kappa(\beta^*,0,5\beta^*,\sigma)^t\beta^* + \frac{1}{1-\kappa(\beta^*,0.5\beta^*,\sigma)}O\left(\sqrt{\frac{(\beta^*+C_f\sigma)^{2}}{n}\log \frac{1}{\delta}}\right)\leq \epsilon \beta^*,
\end{align*}
so that the iterates get $\epsilon$-close to $\beta^*$. Again, the assumption on the sample complexity per iteration in  Proposition~\ref{prop:second_stage} guarantees that the first part and the second part are both bounded by $0.5\epsilon \beta^*$. The proof is similar as before. 
\end{proof}

\subsubsection{Supporting Lemmas for Proposition~\ref{prop:first_stage} and Proposition~\ref{prop:second_stage}}
From the one-step analysis in the finite sample case as established in Proposition \ref{prop:1d-finite-analysis}, we would like to determine the region of contraction to $\beta^*$ (i.e, $|\tilde{\beta}^+-\beta^*|<|\beta-\beta^*|$) with probability at least $1-\delta$:
\begin{align}
\label{eq:contraction-region}
\mathcal{C}(f,\beta^*,\sigma):=\left\{\beta: (1-\kappa(\beta^*,\beta,\sigma ))|\beta-\beta^*|> O\left(\sqrt{\frac{C_f(\beta^*,\sigma)^{2}}{n}}\log \frac{1}{\delta}\right)\right\}.
\end{align}
This region allows us to control the convergence rate for the iterates. 
 Using the fact that $\kappa(\beta^*,\beta,\sigma)$ depends on $\min(\beta,\beta^*)$, a more explicit condition for the contraction region is the following:
\begin{align}
    \label{eq:condition_small}
    \left(1-\kappa(\beta^*,\beta,\sigma)\right)|\beta-\beta^*|\geq O\left(\sqrt{\frac{C_f(\beta^*,\sigma)^{2}}{n}}\log \frac{1}{\delta}\right)\quad \text{when }\beta < \beta^*;\\
    \label{eq:condition_big}
    \left(1-\kappa(\beta^*,\beta^*,\sigma)\right)|\beta-\beta^*|\geq O\left(\sqrt{\frac{C_f(\beta^*,\sigma)^{2}}{n}}\log \frac{1}{\delta}\right)\quad \text{when }\beta > \beta^*.
    \end{align}

Note that in \eqref{eq:condition_small}, $\beta$ being close to $0$ or close to $\beta^*$ will make the left hand side vanish, thus we infer that the contraction region for $\beta\leq \beta^*$ is an open interval $(L_1,L_2)$. In \eqref{eq:condition_big}, we infer that the contraction region for $\beta>\beta^*$ is an open interval $(R_1,\infty)$. We provide an illustration in Figure \ref{fig:contract_region}:
\begin{figure}[h]
    \centering
    \includegraphics[scale=0.4, trim = 0 300 100 300, clip]{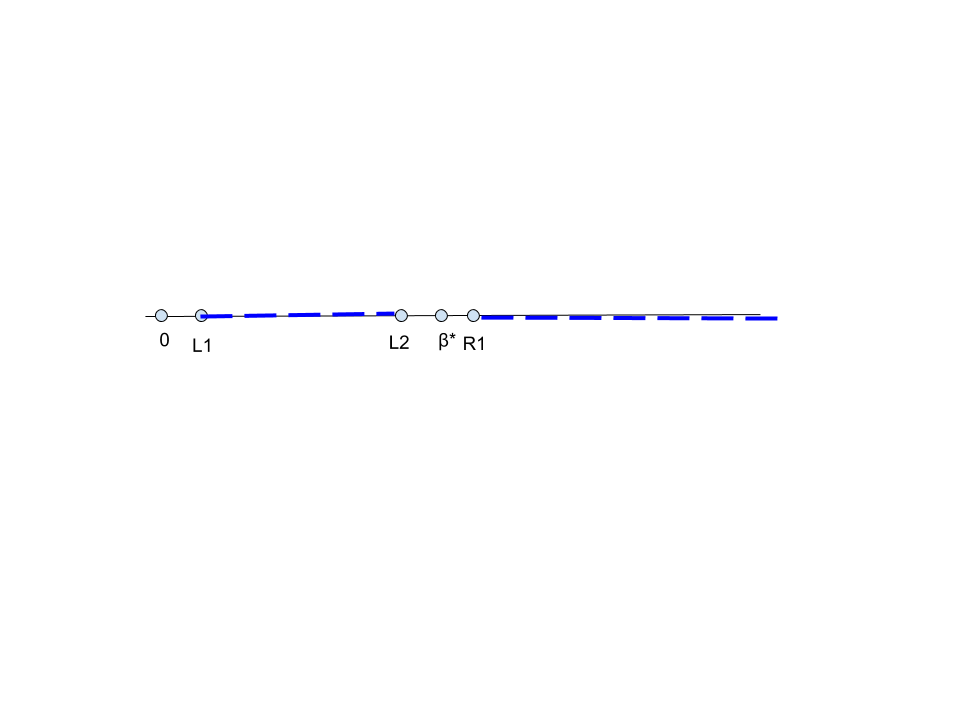}
    \caption{Contraction region: $(L_1,L_2)$ and $(R_1,\infty)$}
    \label{fig:contract_region}
\end{figure}

\begin{lemma}[Contraction implies Stability of the Iterate]
\label{lem:stability}
Suppose that $\hat{\beta} \in \mathcal{C}(f,\beta^*,\sigma)$, the contraction region, and $\hat{\beta}<\beta^*$. For all $\beta \in (\hat{\beta},2\beta^*-\hat{\beta})$, we have $\tilde{\beta}^+ \in (\hat{\beta},2\beta^*-\hat{\beta})$.
\end{lemma}
\begin{proof}
Using Proposition \ref{prop:1d-finite-analysis}, we have
\begin{align}
|\tilde{\beta}^+-\beta^{*}|\leq & \kappa(\beta^*,\beta,\sigma)|\beta-\beta^{*}|+O\left(\sqrt{\frac{C_f(\beta^*,\sigma)^{2}}{n}\log \frac{1}{\delta}}\right) \nonumber\\
\leq & \kappa(\beta^*,\hat{\beta},\sigma)|\hat{\beta}-\beta^{*}|+O\left(\sqrt{\frac{C_f(\beta^*,\sigma)^{2}}{n}\log \frac{1}{\delta}}\right) \label{eq:change_with_vemhat}\\
\leq & |\hat{\beta}-\beta^*|, \label{eq:assumption_contraction}
\end{align}
where \eqref{eq:change_with_vemhat} follows from $\kappa(\beta^*,\beta,\sigma)\leq \kappa(\beta^*,\hat{\beta},\sigma)$ and $|\beta - \beta^*|\leq |\hat{\beta}-\beta^*|$. Step \eqref{eq:assumption_contraction} follows by the assumption that $\hat{\beta}\in \mathcal{C}(f,\beta^*,\sigma)$ and \eqref{eq:condition_small}.
\end{proof}

\begin{lemma}
\label{lem:l1}
For every $\ell\in (0,0.5\beta^*)$, suppose that $n=\tilde{\Omega}\left(\frac{(1+C_f/\eta)^2}{(1-\kappa(\beta^*,\ell,\sigma))^2}\right)$, we have that both $\ell$ and $0.5\beta^*$ are in $\mathcal{C}(f,\beta^*,\sigma).$ 
\end{lemma}
\begin{proof}
For $\ell<0.5\beta^*$ to be in the contraction region, a sufficient condition is the following:
\[
0.5\beta^* \left(1-\kappa(\beta^*,\ell,\sigma)\right)\geq O\left(\sqrt{\frac{C_f(\beta^*,\sigma)^{2}}{n}\log \frac{1}{\delta}}\right).
\]
The above condition also implies that $0.5\beta^*\in \mathcal{C}(f,\beta^*,\sigma)$ since $\kappa(\beta^*,\ell,\sigma)>\kappa(\beta^*,0.5\beta^*,\sigma)$.
\end{proof}

\begin{corollary}
\label{cor:greater_than_ell}
Let $\ell\in (0,0.5\beta^*)$. Suppose that $n=\tilde{\Omega}\left(\frac{(1+C_f/\eta)^2}{(1-\kappa(\beta^*,\ell,\sigma))^2}\right)$, we have that for all $\beta \in (\ell,\infty)$, $\tilde{\beta}^+ \in (\ell,\infty)$.
\end{corollary}
\begin{proof}
By Lemma \ref{lem:stability} and Lemma \ref{lem:l1}, it suffices to consider $\beta>2\beta^*-\ell$. Recall the property of population \vem update for $\beta^+$: $\beta^+>\beta^*$. Also recall the intermediate result from Proposition~\ref{prop:1d-finite-analysis}, which ensures that with probability at least $1-\delta$,
\[
|\tilde{\beta}^+ - \beta^+|\leq O\left(\sqrt{\frac{(\beta^*+C_f\sigma)^2}{n}\log \frac{1}{\delta}}\right).
\]
It follows that 
\begin{align*}
    \tilde{\beta}^+ \geq & \beta^+ - O\left(\sqrt{\frac{(\beta^*+C_f\sigma)^2}{n}\log \frac{1}{\delta}}\right) \\
    > & \beta^* - O\left(\sqrt{\frac{(\beta^*+C_f\sigma)^2}{n}\log \frac{1}{\delta}}\right) >0.5\beta^*>\ell.
\end{align*}
This completes the proof of the corollary.
\end{proof}

\subsection{Finite-sample convergence guarantees for special cases}
\label{subsec:finite-special-case}
We have shown in Section \ref{subsec:concrete-convergence-rate} that for Gaussian, Laplace and logistic distribution, the contraction factor takes the form $\kappa(\beta^*,\beta,\sigma )=\exp\left(-c\frac{\min(\beta,\beta^*)^{\gamma_f}}{\sigma^{\gamma_f}}\right)$, for some $\gamma_f\geq 1$ determined by the asymptotic growth of the log density.

In view of Propositions~\ref{prop:first_stage} and~\ref{prop:second_stage}, we deduce the following overall convergence result:
\begin{corollary}[Explicit Convergence Rate]
Suppose that log concave density $f$ satisfies the regularity condition, and the contraction ratio $\kappa(\beta^*,\beta,\sigma) = \exp\left(-\frac{\min(\beta,\beta^*)^{\gamma_f}}{\sigma^{\gamma_f}}\right)$. We run the \vem algorithm with a an initial point $\beta^0$. If $\beta^0$ falls in the local region of $\beta^*$:$(0.5\beta^*,1.5\beta^*)$, the \vem algorithm outputs a solution $\tilde{\beta}^T$ such that $|\tilde{\beta}^T-\beta^*|\leq \epsilon \beta^*$ after $T = O\left(\log\epsilon / \eta^{\gamma_f}\right)$ iterations. The number of fresh samples required per iteration is: $N/T=\tilde{\Omega}\left(\frac{(\beta^*+C_f/\eta)^2}{\epsilon^2\eta^{2\gamma_f}}\right)$. Otherwise, if the initial point $\beta^0\in (0,0.5\beta^*)$ or $(0.5\beta^*,\infty)$, the \vem algorithm will take an additional $T^\prime = O\left(\frac{\log \frac{0.25\beta^*}{|\beta^0-\beta^*|}}{\eta^{\gamma_f}}\right)$ iterations before the iterates enter the local region $(0.5\beta^*,1.5\beta^*)$. The number of fresh samples required per iteration is $\tilde{\Omega}\left(\frac{(1+C_f/\eta)^2}{\eta^{2\gamma_f}}\right)$.
\end{corollary}

\section{Model Mis-specification}
\label{appendix-subsec:mis}
In this section, we establish the robustness results for the \vem algorithm with a mis-specified distribution in 1-D.  Proposition \ref{lem:mis-gaussian} is proved in Section \ref{subsec:proof-mis-gaussian}. In Section \ref{subsec:mis-observation}, we present some numerical observations for the robustness of the \vem algorithm.
\subsection{Proof of Proposition \ref{lem:mis-gaussian}}
\label{subsec:proof-mis-gaussian}

\begin{lemma}[3 fixed points when misspecified]
\label{lem:mis-3-fixed}
Suppose that $f\in \mathcal{F}$ satisfy the regularity condition. We further assume in the region $\beta\geq 0$, the function
$\hat{F}_{\beta,\sigma}(x):=\hat{g}\left(\frac{1}{\sigma}|x+\beta|\right)-\hat{g}\left(\frac{1}{\sigma}|x-\beta|\right)$
is a concave function in $\beta$ for each $x\geq 0$, and
\[
\mathbb{E}_{X\sim f_{\beta^*}}X\hat{g}^{\prime}(X)>1.
\]
The iterates of the \vem algorithm with mis-specified log-concave density $\hat{g}$ converge to a non-zero $\bar{\beta}$ ( or $-\bar{\beta})$ from a non-zero random initialization. 
\end{lemma}

\begin{proof}
With mis-specified log density $\hat{g}$, the next iterate is:
\begin{align*}
    \hat{\beta}^+ = & \mathbb{E}_{X\sim f_{\beta^*,\sigma}} X\tanh(0.5\hat{F}_{\beta,\sigma}(X))\\
    = & \int_{x\geq 0} \Big(\frac{1}{\sigma}f\left(\frac{x-\beta^*}{\sigma}\right)+\frac{1}{\sigma}f \left(\frac{x+\beta^*}{\sigma}\right)\Big)x\tanh(0.5\hat{F}_{\beta,\sigma}(x)) \dup x\\
    =: &\hat{M}(\beta^*,\beta).
\end{align*}
We state the key properties for the function $\hat{M}(\beta^*,\cdot)$:
\begin{itemize}
    \item $\hat{M}(\beta^*,\cdot)$ is an increasing function in $\beta$: we utilize the convexity of $\hat{g}$, similar to the proof of Lemma~\ref{lem:increasing_wrt_beta};
    \item $\hat{M}(\beta^*,\cdot)$ is a concave function in $\beta$: since $\hat{F}_{\beta,\sigma}$ is a concave function in $\beta$ for every $x\geq 0$, so is $\tanh(0.5\hat{F}_{\beta,\sigma})$ which is a composition with a concave function $\tanh$ that is increasing on the non-negative part. $\hat{M}(\beta^*,\cdot)$ is concave in $\beta$ since it is an intergal of concave functions of $\beta$;
    \item $\hat{M}(\beta^*,0)=0$;
    \item $\frac{\partial}{\partial \beta}\hat{M}(\beta^*,\beta)\mid_{\beta=0}>1$: the assumption that $\mathbb{E}_{X\sim f_{\beta^*}}X\hat{g}^{\prime}(X)\geq 1$ is equivalent to $\frac{\partial \hat{M}(\beta^*,\beta)}{\partial \beta}\mid_{\beta=0}>1$;
    \item $\hat{M}(\beta^*,\beta)-\beta\to -\infty$ as $\beta\to \infty$: since $\hat{M}(\beta^*,\beta)$ is bounded.
\end{itemize}

The above three properties guarantees that on $(0,\infty)$, $\hat{M}(\beta^*,\cdot)$ has a unique fixed point $\bar{\beta}>0$ satisfying the following properties: (1)$\hat{M}(\beta^*,\bar{\beta}) = \bar{\beta}$; (2) If $\beta\in (0,\bar{\beta})$, $\hat{M}(\beta^*,\bar{\beta})\in (\beta,\bar{\beta})$; (3) If $\beta \in (\bar{\beta},\infty)$, $\hat{M}(\beta^*,\bar{\beta})\in (\bar{\beta},\beta)$. Since $\hat{M}(\beta^*,\cdot)$ is an odd function in $\beta$, we conclude that on $(-\infty,0)$, $\hat{M}(\beta^*,\cdot)$ has a unique fixed point $-\bar{\beta}$. In view of the above properties, we deduce that if an initial point is positive, it converges to $\bar{\beta}$; and if it is negative, it converges to $-\bar{\beta}$. 
\end{proof}

\begin{lemma}[Error bound when misspecified]
\label{lem:error_misspecified}
Suppose that the assumption in Lemma \ref{lem:mis-3-fixed} holds and the function $x\tanh(0.5\hat{F}_{\beta,\sigma}(x))$ is $L$-Lipschitz. For any $\beta^{0}\neq 0$, the \vem
with misspecified log concave distribution $\hat{f}$ will converge
to a solution $\bar{\beta}$ ,and 
\[
|\bar{\beta}-\textup{sign}(\beta^{0},\beta^{*})\beta^{*}|\leq\frac{6\sigma}{1-\kappa(\bar{\beta},\beta^{*},\sigma)},
\]
where $\kappa(\bar{\beta},\beta^{*},\sigma)\in(0,1)$ is defined in Theorem \ref{thm:Global-Convergence-1D-VEM}.
\end{lemma}

\begin{proof}
The fixed point structure established in Lemma~\ref{lem:mis-3-fixed} ensures that the iterates converge to either $\bar{\beta}$ or $-\bar{\beta}$, depending on the sign of $\beta^0$. 
$\bar{\beta}$ satisfies $\mathbb{E}_{X\sim f_{\beta^*,\sigma}}X\tanh(0.5\hat{F}_{\bar{\beta},\sigma}(X)) = \bar{\beta}$. We can decompose the difference between $\bar{\beta}$ and $\beta^*$ in the following way (utilizing the consistency property of the \vem update):
\begin{align*}
\bar{\beta}-\beta^{*}= & \mathbb{E}_{X\sim f_{\beta^*,\sigma}}X\tanh(0.5\hat{F}_{\bar{\beta},\sigma}(X))-\mathbb{E}_{X\sim \hat{f}_{\beta^*,\sigma}} X\tanh(0.5\hat{F}_{\beta^*,\sigma}(X))\\
= & \underbrace{\mathbb{E}_{X\sim f_{\beta^*,\sigma}}X\tanh(0.5\hat{F}_{\bar{\beta},\sigma}(X))-\mathbb{E}_{X\sim \hat{f}_{\beta^*,\sigma}}X\tanh(0.5\hat{F}_{\bar{\beta},\sigma}(X))}_{A}\\
 & +\underbrace{\mathbb{E}_{X\sim \hat{f}_{\beta^*,\sigma}}X\tanh(0.5\hat{F}_{\bar{\beta},\sigma}(X))-\mathbb{E}_{X\sim \hat{f}_{\beta^*,\sigma}}X\tanh(0.5\hat{F}_{\beta^*,\sigma}(X))}_{B}.
\end{align*}
Let us control $A$ and $B$ separately. The term $B$ is exactly the difference between a \variant EM update and
the true location parameter with the log-concave distribution $\hat{f}$.
Therefore, Theorem~\ref{thm:Global-Convergence-1D-VEM} tells us that
\[
|B|\leq\kappa_{\hat{f}}(\bar{\beta},\beta^{*},\sigma)|\beta-\beta^{*}|.
\]
for some $\kappa_{\hat{f}}(\bar{\beta},\beta^{*},\sigma)\in (0,1)$. 
For term $A$, we note that the integrand is $L$-Lipschitz (by
the assumption), it can be bounded by the Wasserstein distance 
between two distributions as follows:
\[
A\leq L\cdot D_{W}\left(\frac{1}{\sigma}f\left(\frac{1}{\sigma}(\cdot-\beta^{*})\right),\frac{1}{\sigma}\hat{f}\left(\frac{1}{\sigma}(\cdot-\beta^{*})\right)\right).
\]
Here we use $D_{W}$ to denote the Wasserstein distance. By scaling and translation, we have:
\[
D_{W}\Bigg(\frac{1}{\sigma}f\left(\frac{1}{\sigma}\left(\cdot-\beta^{*}\right)\right),\frac{1}{\sigma}\hat{f}\left(\frac{1}{\sigma}\left(\cdot-\beta^{*}\right)\right)\Bigg)=\sigma D_{W}(f,\hat{f}),
\]
where $f$ and $\hat{f}$ are two log-concave distribution with unit
variance. Using the triangle inequality, it can be further bounded
by:
\begin{align*}
D_{W}(f,\hat{f}) & \leq D_{W}(f,\mathcal{N}(0,1))+D_{W}(\hat{f},\mathcal{N}(0,1))\\
 & \leq2\sup_{f\in\mathcal{F}}D_{W}(f,\mathcal{N}(0,1)).
\end{align*}
Now we can apply the classical Stein's method to bound the Wasserstein distance; in particular, we apply Proposition~\ref{prop:Stein}  and obtain that:
\begin{align*}
 D_{W}(f,\mathcal{N}(0,1))\leq & |\mathbb{E}_{W\sim f}Wh(W)-h^{\prime}(W)|\\
\leq & 2\mathbb{\mathbb{E}}_{W\sim f}|W|+\sqrt{\frac{\pi}{2}}\\
\leq & 2\sqrt{\mathbb{E}_{W\sim f}W^{2}}+\sqrt{\frac{\pi}{2}}\leq3.
\end{align*}
The last line follows since we assume $W\sim f$ has unit variance. Combining the bound on $A$ and $B$ together, we have proved the following:
\begin{align*}
|\bar{\beta}-\beta^{*}|\leq \kappa_{\hat{f}}(\bar{\beta},\beta^*,\sigma) |\bar{\beta}-\beta^{*}| + 6L\sigma,    
\end{align*}
and rearranging the inequality yields:
\begin{align}
\label{eq:error_bound}
|\bar{\beta}-\beta^{*}|\leq\frac{6L\sigma}{1-\kappa_{\hat{f}}(\bar{\beta},\beta^{*},\sigma)}.
\end{align}
This completes the proof. 
\end{proof}

\paragraph{Fitting with Gaussian}
Now let us consider a special case where $\hat{f}$ is the Gaussian distribution. The \textit{misspecified} \vem update is the following:
\begin{align}
\label{eq:mis2gmm}
    \hat{M}(\beta^*,\beta)=\mathbb{E}_{X\sim f_{\beta^*,\sigma}} X\tanh\left(\frac{\beta X}{\sigma^2}\right)
\end{align}
The conditions in Lemma \ref{lem:error_misspecified} satisfies. Furthermore, we can prove a lower bound for $\bar{\beta}$ in \eqref{eq:error_bound} when the SNR is high. This allows us to obtain a better error bound for controlling the distance between $\bar{\beta}$ and $\beta^*$. This is the content of Proposition~\ref{lem:mis-gaussian}, restated below.

\misspecifiedgmm*

\begin{proof}
The gradient with respect to $\beta=0$ for $\hat{M}(\beta^*,\beta)$ is $1+\frac{(\beta^*)^2 }{\sigma^2}>1$. Meanwhile, it is easy to see that  $\hat{M}(\beta^*,\beta)$ is a concave function of $\beta$ on the region where $\beta \geq 0$. Moreover, $x\tanh(\frac{\beta x}{\sigma^2})$ is $1.5$-Lipschitz as a function of $x$ for all $\beta\geq 0$.  Lemma \ref{lem:mis-3-fixed} is applicable and it tells us that the mis-specified \vem updates converge to a point $\bar{\beta}$ from a random non-zero initialization. 

In the following: let $M_g(\beta^*,\beta)$ denote the \variant EM update with the ground truth log-concave distribution: $0.5\mathcal{N}(\beta^*,\sigma^2)+0.5\mathcal{N}(-\beta^*,\sigma^2)$. In Lemma~\ref{lem:error_misspecified}, we have proved an intermediate result (bound for term $B$) that:
\begin{align}
\label{eq:mis-gap}
   |\hat{M}(\beta^*,\beta) - M_g(\beta^*,\beta)|\leq 9\sigma,
\end{align}

Meanwhile, the convergence result from \cite{daskalakis2016ten} (i.e., Corollary~\ref{cor:gaussian}) says that at $\beta = \frac{\beta^*}{2}$, 
\begin{align}
\label{eq:10-step-bound}
    \beta^* -M_g(\beta^*,0.5\beta^*)\leq 0.5 \exp\left(-\frac{(\beta^*)^2}{8\sigma^2}\right) \beta^*.
\end{align}
Combining \eqref{eq:mis-gap} and \eqref{eq:10-step-bound}, We deduce  that:
\begin{align*}
\hat{M}(\beta^*,0.5\beta^*) \geq & M_g(\beta^*,0.5\beta^*)-9\sigma \\
\geq &\beta^*\Big(1-0.5 \exp\left(-\frac{(\beta^*)^2}{8\sigma^2}\right)\Big)- 9\sigma\\
= & \beta^*\left(1-0.5\exp(-0.125\eta^2)-\frac{1}{\eta}\right).
\end{align*}
When $\eta >C_0$ for some absolute constant $C_0>0$, we can show that $\beta^*\left(1-0.5\exp(-0.125\eta^2)-\frac{1}{\eta}\right)>0.5\beta^*$.  In particular, this implies that $\bar{\beta}>0.5\beta^*$ by Lemma \ref{lem:mis-3-fixed}. Therefore, the error bound in equation~\eqref{eq:error_bound} can be further bounded by
\begin{align*}
|\bar{\beta}-\beta^{*}|\leq\frac{9\sigma}{1-\exp\left(-0.125\eta^2\right)}.   
\end{align*}
The right hand side bound is smaller than $10\sigma$ when $\eta$ is large. This completes the proof of Proposition~\ref{lem:mis-gaussian}.
\end{proof}

\begin{proposition}[Wasserstein Distance Bound by Stein's Method~\cite{ross2011fundamentals}]
\label{prop:Stein}
We have
\[
D_{W}(f,\mathcal{N}(0,1))\leq\sup_{h\in\mathcal{F}}|\mathbb{E}\left[Wh(W)-h'(W)\right]|,
\]
where $\mathcal{F=}\{h:\|h\|\leq2,\|h^{\prime}\|\leq\sqrt{\frac{\pi}{2}},\|h^{\prime\prime}\|\leq2\}$. 
\end{proposition}

\subsection{General Observations}
\label{subsec:mis-observation}
In Section \ref{subsec:proof-mis-gaussian}, the robustness results rely on the assumptions in Lemma \ref{lem:mis-3-fixed}. In particular, we need the concavity of $\hat{F}_{\beta,\sigma}$. This is a very restrictive condition. Consider the family of the log-concave distribution whose log density is of the form $g(x)\propto |x|^r$, $r\geq 1$. The concavity condition holds only when $r\leq 2$. However, the 3-fixed point structure still holds as along as the five properties in the proof for Lemma \ref{lem:mis-3-fixed} hold. Indeed we observe that the function $\beta \to \hat{M}(\beta^*,\beta)$ is a concave function in $\beta$ even when $r>2$ (see Section~\ref{appendix:numerical-mis}).

Recall that the \variant EM iterate with the misspecified distribution $\hat{f}$ is the following:
\begin{align*}
\hat{M}(\beta^*,\beta) = & \mathbb{E}_{X\sim f_{\beta^*,\sigma}} x\tanh(0.5\hat{F}_{\beta,\sigma}(X)),
\end{align*}
where 
\begin{align*}
\hat{g} =  \log \hat{f} 
\quad\qquad\text{and}\qquad\quad
\hat{F}_{\beta,\sigma}(X) = \hat{g}\left(\frac{1}{\sigma}|X-\beta|\right) -\hat{g}\left(\frac{1}{\sigma}|X+\beta|\right).   
\end{align*}
We can decompose the iterate $\hat{M}(\beta^*,\beta)$ as follows:
\begin{align*}
    \hat{M}(\beta^*,\beta) = & \mathbb{E}_{X\sim f_{\beta^*,\sigma}} x\tanh(0.5\hat{F}_{\beta,\sigma}(x)) - 
     \mathbb{E}_{X\sim \hat{f}_{\beta^*,\sigma}} x\tanh(0.5\hat{F}_{\beta,\sigma}(X)) \nonumber \\
     &+ \mathbb{E}_{X\sim \hat{f}_{\beta^*,\sigma}} x\tanh(0.5\hat{F}_{\beta,\sigma}(X)),
\end{align*}
where the first difference term is a \textit{drift} term specifying the error due to the distribution misspecification and the second term is the ideal \vem update with $\hat{f}$, which contracts to $\beta^*$ at a linear rate. We empirically observe that if we fit a lighter tail log-concave distribution $\hat{f}$ compared to $f$, the drift term is positive (see Section~\ref{appendix:numerical-mis}) and thus the \variant EM iterate converges to some $\bar{\beta}>\beta^*$. By the triangle inequality, 
\begin{align*}
    \hat{M}(\beta^*,\bar{\beta})\leq \mathbb{E}_{X\sim f_{\beta^*,\sigma}} \big| X\tanh(0.5\hat{F}_{\bar{\beta},\sigma}(X)) \big| \leq  \beta^* + \mathbb{E}_{x\sim f_{\beta^*,\sigma}}|X|
    \leq \beta^*+\sigma.
\end{align*}
Thus the relative error is bounded by $\frac{\sigma}{\beta^*} = \frac{1}{\eta}$. When the SNR is large, the error is small. On the other hand, if we fit a heavier tail distribuiton $\hat{f}$ compared to $f$, the drift is negative and the corresponding fixed point can be $0$. This suggests a practical recipe: when one does not know the ground truth log-concave density, fit with a density that has a lighter tail. For instance, we can fit a Gaussian density when the ground truth is Laplace or logistic.

\section{Numerical Experiments}
\label{appendix-numerical}
We provide numerical evidences for the two observations we made in the main sections. In particular, in section \ref{appendix:numerical-approxm}, we demonstrate that the \variant M-step may be viewed as an approximation to the M-step in the classical EM algorithm. In section \ref{appendix:numerical-mis}, we study the convergence behavior as well as the quality of the solution for the \vem with a mis-specified distribution.  

\subsection{Approximate M-step}
\label{appendix:numerical-approxm}

We consider the family of polynomial distributions with log density $g\propto |x|^r$ for some $r\geq 1$. In the E-step of the classical EM algorithm, we obtain a lower bound $Q(\cdot \mid \beta)$ for the log-likelihood based on the current estimate $\beta$: 
\begin{align}
    Q(b\mid \beta) = \mathbb{E}_{x\sim f_{\beta^*,\sigma}}\left[-p_{\beta,\sigma}^1 g\left(\frac{|x-b|}{\sigma}\right) - p_{\beta,\sigma}^2 g\left(\frac{|x+b|}{\sigma}\right)\right].
\end{align}
The M-step is to compute $\text{argmax}_{b} Q(b\mid \beta)$. Note that $Q(\cdot\mid \beta)$ is a concave function, thus the optimization problem has a well-defined solution; however, it does not admit a closed form solution in general. Consider the above example where the ground truth distribution is polynomial, the M-step is equivalent to solving for a polynomial equation with degree $r-1$.

In the following, we plot the negative $Q$ function (which is convex) for two polynomial distributions in Figure~\ref{fig:q1} and Figure~\ref{fig:q2}. Meanwhile  we trace two points  $(\beta,Q(\beta\mid \beta))$ and $(\beta^+,Q(\beta^+\mid \beta))$, where $\beta^+$ is the \vem update:
\begin{align*}
\beta^+ = \mathbb{E}_{x\sim f_{\beta^*,\sigma}}\left[x\tanh(0.5F_{\beta,\sigma}(x))\right].
\end{align*}
Numerically we find that $b=\beta^+$ strictly increases the value of $Q$ function compared to $b=\beta$ when $\beta$ is not equal to the true parameter. 
\begin{figure}[h]
    \centering
    \includegraphics[scale=0.12, trim =  100 100 300 200, clip]{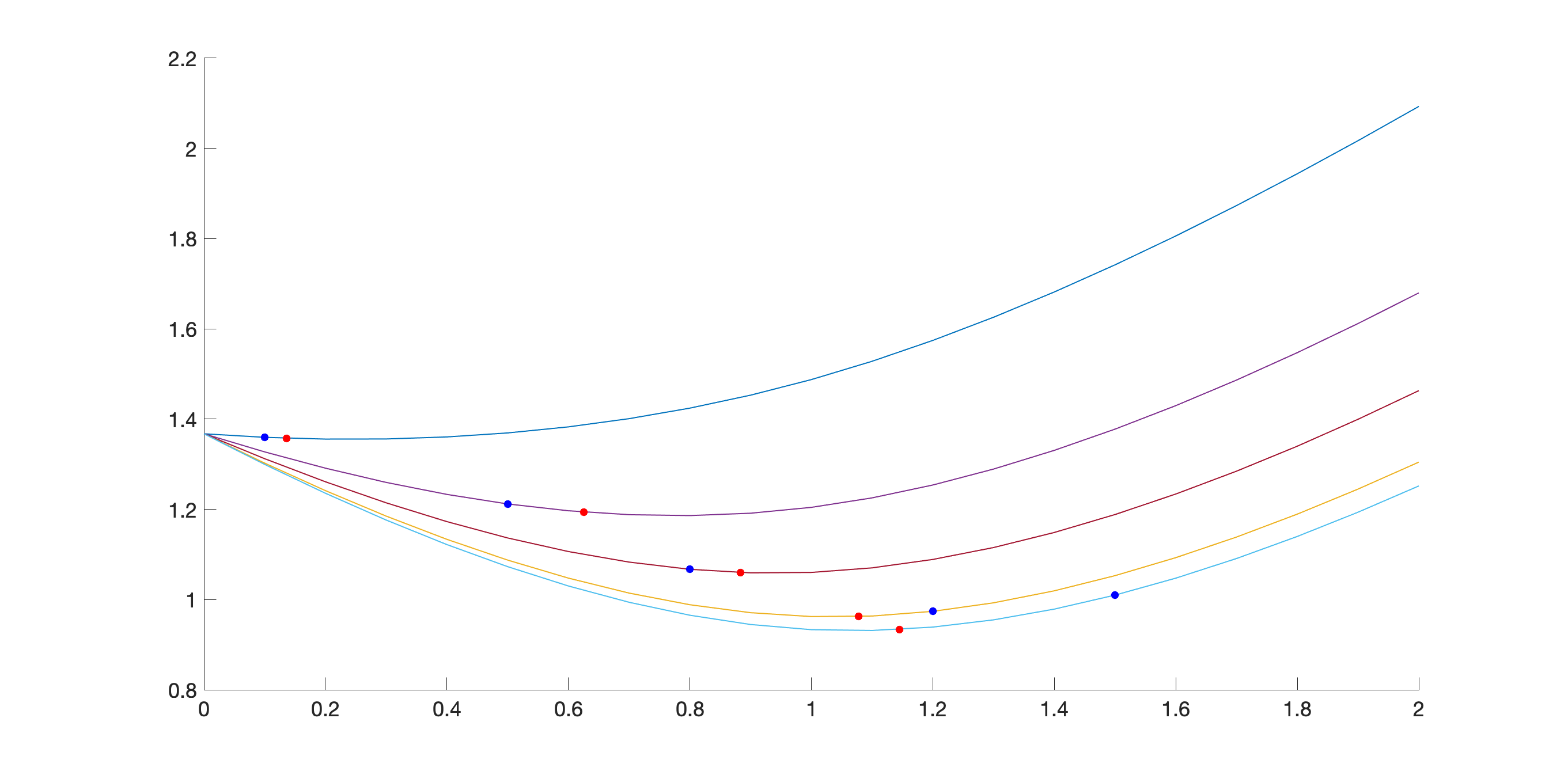}
    \caption{Plot of the function: $-Q(\cdot \mid \beta)$ for $g\propto |x|$ for $\beta=0.1,0.5,0.8,1.2,1.5$, represented by different colored curves. The true location parameter $\beta^*=1$. The blue dots correspond to $(\beta,-Q(\beta \mid \beta))$ and the red dots correspond to $(\beta^+, -Q(\beta^+ \mid \beta))$. It is seen that $\beta^+$ is not the exact the M-step, as they do not minimize the $-Q(\cdot \mid \beta)$ function. However, it is seen that $-Q(\beta^+\mid \beta)< -Q(\beta\mid \beta)$, suggesting that the \vem update is a type of approximate M-step.
    \yc{Scaling/trimming seems off. Also, what does each color represent? Also explain in this caption what are the red/blue dots.}}
    \label{fig:q1}
\end{figure}
\begin{figure}[h]
    \centering
    \includegraphics[scale=0.12,trim = 100 100 300 200, clip]{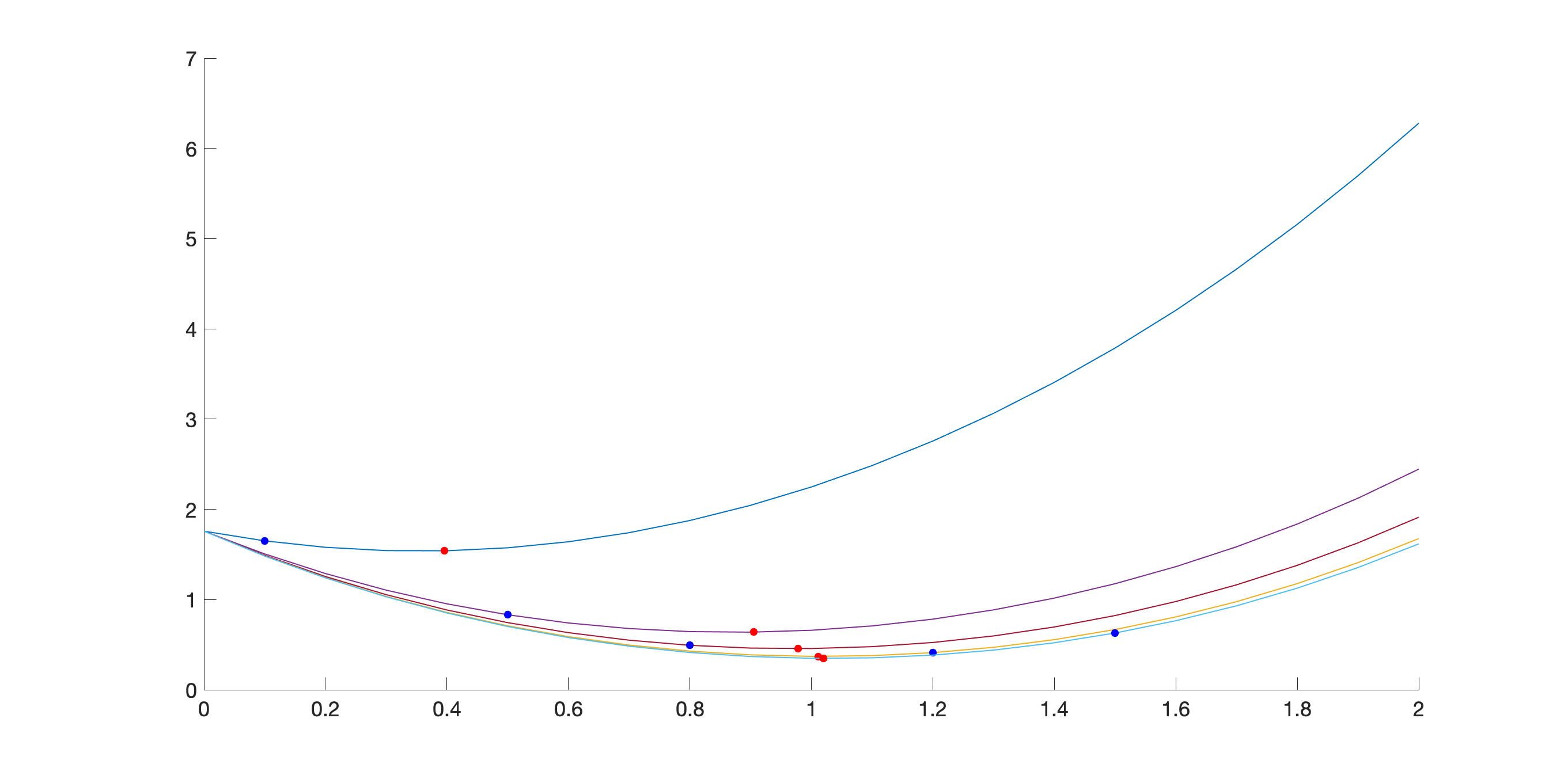}
    \caption{Plot of $-Q(\cdot \mid \beta)$ for $g\propto |x|^{2.5}$ for $\beta=0.1,0.5,0.8,1.2,1.5$,  represented by different colored curves. The true parameter is $\beta^*=1$. Similar to Figure \ref{fig:q1}, the blue dots correspond to $(\beta,-Q(\beta \mid \beta))$ and the red dots correspond to $(\beta^+, -Q(\beta^+ \mid \beta))$
     \yc{what does each color represent? Also explain in this caption what are the red/blue dots.}
	}
    \label{fig:q2}
\end{figure}

\subsection{Misspecified \vem}
\label{appendix:numerical-mis}

We consider mixtures of polynomial distributions with $\log f = g\propto |x|^r$ for some $r\geq 1$. We pick some $r_0$ in the family as the ground truth distribution, and then fit with another distribution with $r_1$. 

In Figure \ref{fig:mis1} and Figure \ref{fig:mis2}, we plot the new iterate $ \beta^+ $ computed by \vem as a function of the current iterate $ \beta $.
We observe that when $r_1>r_0$, the \vem updates tend to converge to a point greater than $\beta^*$. On the other hand, when $r_1<r_0$, the \vem updates tend to converge to a point smaller than $\beta^*$.  Specifically, in Figure \ref{fig:mis2}, we observe that when fitting a distribution with heavier tail than the ground the truth, the convergence point can be $0$, which might lead to a big error in estimation. Therefore, it suggests that one should fit a distribution with a lighter tail in practice. In both sets of experiments, we observe that fitting a 2GMM yields a fixed point close to the ground truth parameter. 
\begin{figure}
    \centering
    \includegraphics[scale=0.12,trim = 200 0 300 50, clip]{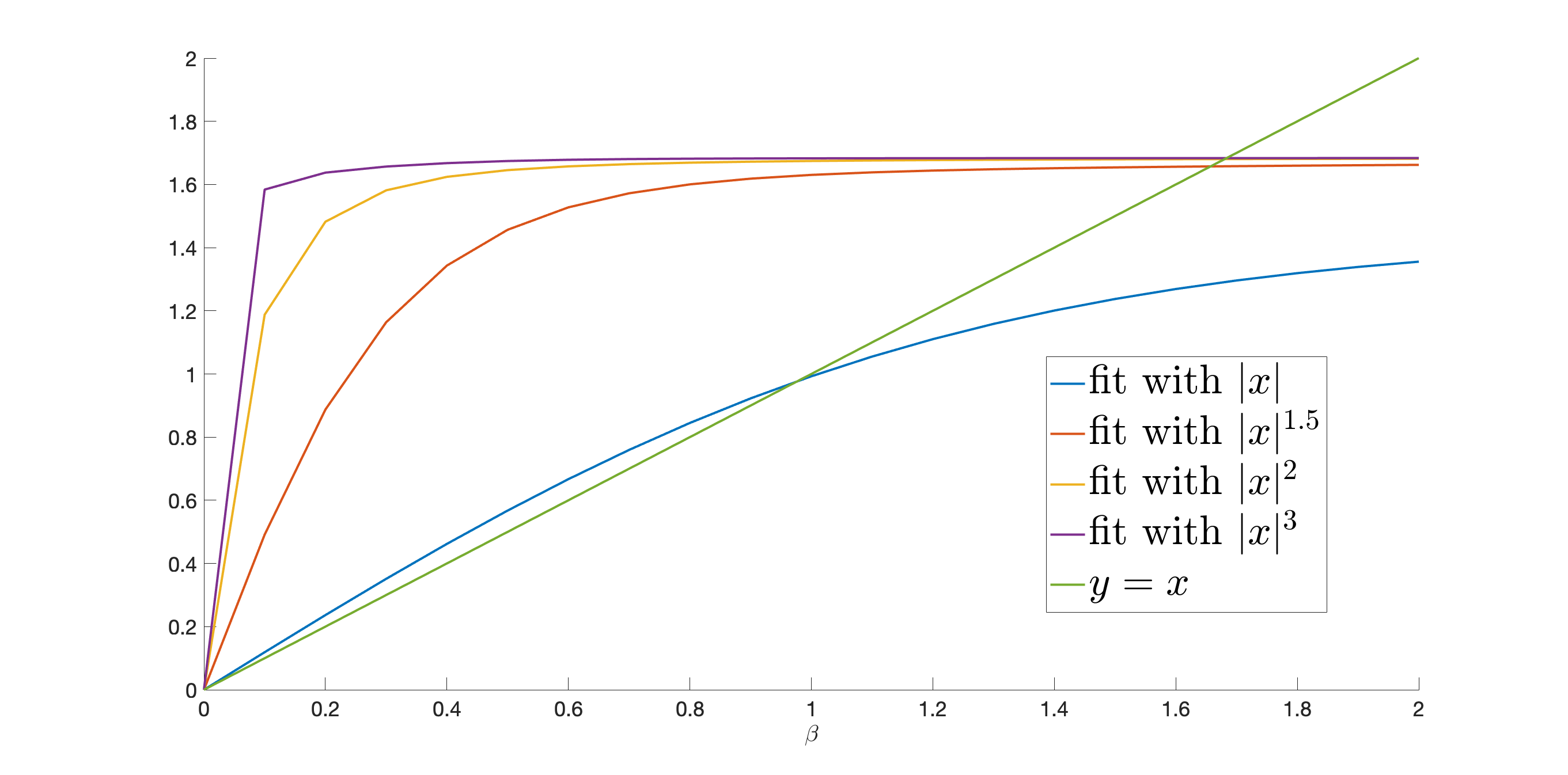}
    \caption{We plot $\beta^+$ as a function of $\beta$.  The ground truth distribution is $g\propto |x|$, with $\beta^*=1$. The mis-specified distributions are picked with degree $1.5,2,3$. The intersection point with the $ y=x $ line is the fixed point for the \vem update. It is seen that when we fit with a polynomial distribution with higher degree, the fixed points are all greater than $1$.}
    \label{fig:mis1}
\end{figure}

\begin{figure}[h]
    \centering
    \includegraphics[scale=0.12,trim = 200 0 300 50, clip]{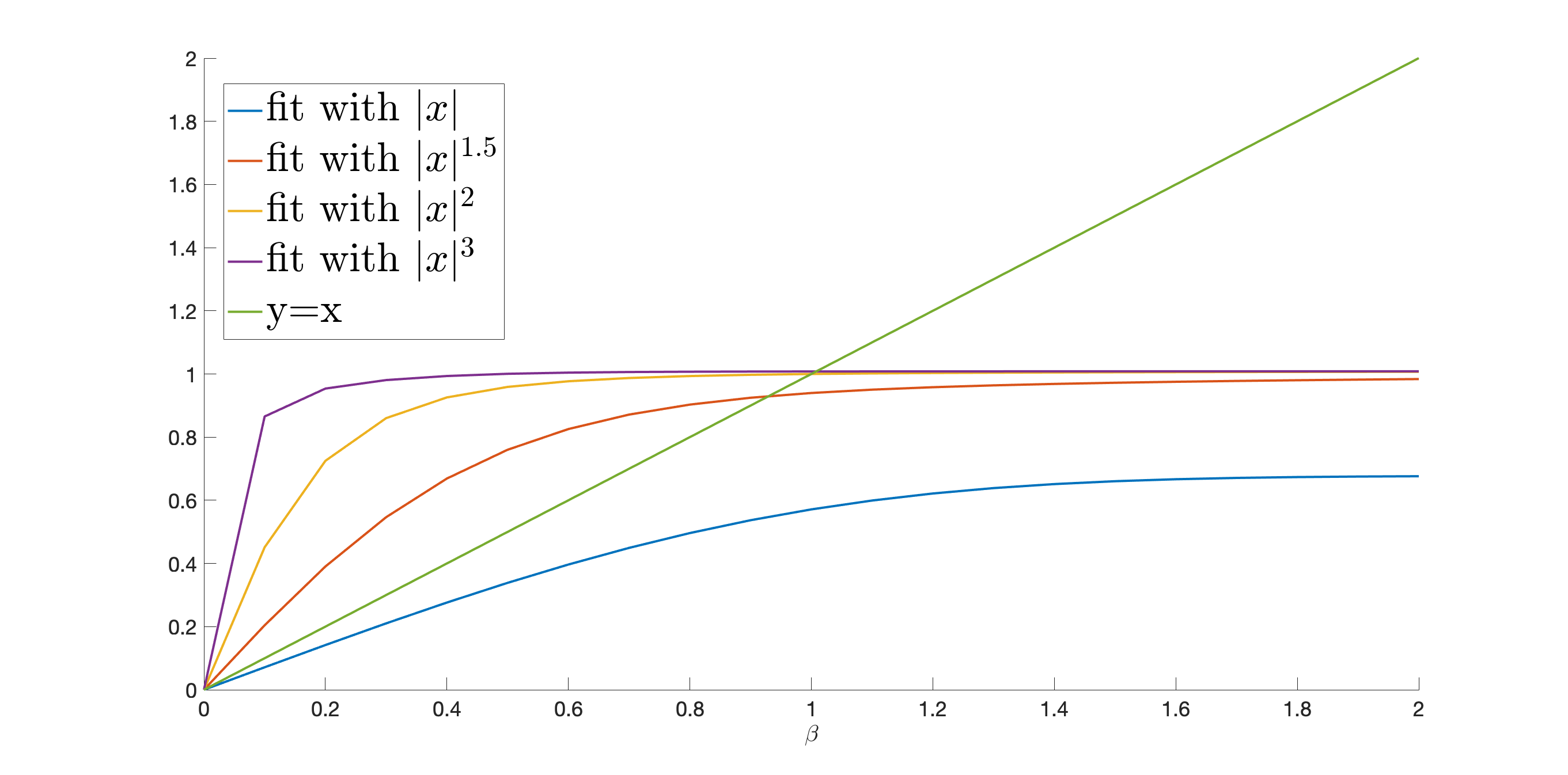}
    \caption{We plot $\beta^+$ as a function of $\beta$. The ground truth distribution is $g\propto |x|^2$, with $\beta^*=1$. The misspecified distributions are picked with degree $1.5,2,3$.  It is seen that when we fit with a polynomial distribution with higher degree, the fixed points are all greater than $1$. When we fit with a polynomial distribution with lower degree, the fixed points are all smaller than $1$. In particular, when we fit with a Laplace distribution, the only fixed point is $0$.
	}
    \label{fig:mis2}
\end{figure}

\end{document}